\documentclass[journal]{IEEEtran}
\usepackage{cite}
\ifCLASSINFOpdf
\else
\fi
\usepackage[cmex10]{amsmath}
\usepackage{amssymb}
\usepackage{amsthm}
\usepackage{algorithmic}
\usepackage{array}

\ifCLASSOPTIONcompsoc
  \usepackage[caption=false,font=normalsize,labelfont=sf,textfont=sf]{subfig}
\else
  \usepackage[caption=false,font=footnotesize]{subfig}
\fi

\usepackage{fixltx2e}
\usepackage{url}
\usepackage{microtype}
\usepackage{graphicx}
\usepackage{epstopdf}
\usepackage{tikz}
\usepackage{pgfplots}
\usetikzlibrary{decorations}
\usetikzlibrary{patterns}

\pgfplotsset{compat=newest}
\pgfplotsset{plot coordinates/math parser=false}
\newlength\fheight
\newlength\fwidth
\makeatletter
\newcommand\resetstackedplots{
\pgfplots@stacked@isfirstplottrue
\addplot [forget plot,draw=none] coordinates{(1,0) (2,0) (3,0) (4,0) (5,0) (6,0)};
}

\usepackage{booktabs}
\usepackage{color}
\usepackage[normalem]{ulem}

\newtheorem{theorem}{Theorem}
\newtheorem{lemma}{Lemma}
\newtheorem{definition}{Definition}

\newcommand{\I}{\mathcal{I}}
\newcommand{\p}{\Psi}
\newcommand{\pa}{\texttt{conreg1}}
\newcommand{\Tr}{\mathrm{Tr}}
\newcommand{\pb}{\texttt{conreg2}}

\begin{document}
\bstctlcite{IEEEexample:BSTcontrol}
\title{Non-iterative rigid 2D/3D point-set registration using semidefinite programming}

\author{\IEEEauthorblockN{Yuehaw Khoo\IEEEauthorrefmark{1}\IEEEauthorrefmark{2} and \and
Ankur Kapoor\IEEEauthorrefmark{1}}\\
\IEEEauthorblockA{\IEEEauthorrefmark{1}Imaging and Computer Vision, Siemens Research Coorporation, Princeton, NJ 08540, USA}\\
\IEEEauthorblockA{\IEEEauthorrefmark{2}Department of Physics, Princeton University, Princeton, NJ 08540 USA}%
\thanks{Y. Khoo (email: ykhoo@princeton.edu), A. Kapoor (email: ankur.kapoor@siemens.com).}}

\IEEEtitleabstractindextext{%
\begin{abstract}
We describe a convex programming framework for pose estimation in 2D/3D point-set registration with unknown point correspondences. We give two mixed-integer nonlinear program (MINLP) formulations of the 2D/3D registration problem when there are multiple 2D images, and propose convex relaxations for both of the MINLPs to semidefinite programs (SDP) that can be solved efficiently by interior point methods. Our approach to the 2D/3D registration problem is non-iterative in nature as we jointly solve for pose and correspondence. Furthermore, these convex programs can readily incorporate feature descriptors of points to enhance registration results. We prove that the convex programs exactly recover the solution to the MINLPs under certain noiseless condition. We apply these formulations to the registration of 3D models of coronary vessels to their 2D projections obtained from multiple intra-operative fluoroscopic images. For this application, we experimentally corroborate the exact recovery property in the absence of noise and further demonstrate robustness of the convex programs in the presence of noise.
\end{abstract}

}

\maketitle
\IEEEdisplaynontitleabstractindextext
\IEEEpeerreviewmaketitle

\section{Introduction}
%
%
%
%
\IEEEPARstart{R}{igid} registration of two point sets is a classical problem in computer vision and medical imaging of finding a transformation that aligns these two sets. Typically, a point-set registration problem consists of two intertwined subproblems, \emph{pose estimation} and \emph{point correspondence}, where solving one is often the pre-condition for solving the other. A canonical formulation of the rigid point-set registration problem for two point clouds is the following. Let $X,Y\in \mathbb{R}^{d\times m}$ be two point sets in dimension $d$, we want to solve:
\begin{equation}
\label{classical reg}
\min_{P \in \Pi_d^{m\times m}, R \in \mathbb{SO}(d)} \|R X - Y P\|_F^2,
\end{equation}
where $\Pi_d^{m\times m}$ is the set of $m \times m$ permutation matrices, and $R\in \mathbb{SO}(d)$ is the rotation matrix from the special orthogonal group in $d$ dimension. Finding a solution to this problem is difficult as it is a nonconvex, mixed integer nonlinear problem. Another close relative of this problem, namely the 2D/3D point-set registration, assumes a 3D point-set and 2D projections of the 3D point-set. The objective is to find out the pose of the 3D model that gives rise to the 2D  degenerate point-set upon projections. This adds an additional complication to that of the regular point-set registration problem, namely, the loss of depth information. In this paper, we focus on solving forms of 2D/3D point-set registration problem described in Section~\ref{section:Problem statement} with guarantees using our semidefinite programs \texttt{conreg1} and \texttt{conreg2}. We further demonstrate their usefulness in the application of coronary vessel imaging.

From a broader perspective, 2D/3D point-set registration problem arises in numerous medical imaging applications in fields such as neurology~\cite{ hipwell2003neuro}, orthopaedics~\cite{ benameur2003ortho}, and cardiology~\cite{ ruijters2009coronary}. The associated body of literature on 2D/3D registration is expanding rapidly, as is apparent in the thorough review of techniques recently published by Markelj et al.~\cite{ markelj2012review}. Amongst the vast literature on the 2D/3D registration problem, we briefly discuss the techniques most closely related to the one presented in this paper, namely, methods for registration of point-sets. 

As mentioned earlier, a typical point-set registration problem consists of two mutually interlocked subproblems, pose estimation and point correspondence. The key idea in seminal Iterative Closest Point (\texttt{ICP})~\cite{besl1992icp} algorithm is to alternatively solve the two subproblems starting from an initial estimate of pose. For the correspondence subproblem, it uses closeness in terms of Euclidean distance between two points to determine whether they correspond to each other. There are many variants proposed to enhance the original \texttt{ICP}, either to make fast correspondence \cite{rusinkiewicz2001ICP} or to include more features to obtain high quality correspondence \cite{sharp2002icp}. Variants for 2D/3D registration include~\cite{Benseghir2013,rivest2012nonrigid,zheng2006point,kang2014robustness}. Though popular, these methods suffer for common drawbacks, namely, they are all \emph{local} methods and do not guarantee global convergence. Their performances all rely on a good initialization and the spatial configuration (distribution) of 3D points.
In 2D/3D point-set registration, for this type of local optimizers several strategies have been proposed to increase the capture range by the use of multi-resolution pyramids~\cite{lau2006global}, use of stochastic global optimization methods such as differential evolution~\cite{rivest2012nonrigid} or direct search wherein the parameter space is systematically and deterministically searched by dividing it into smaller and smaller hyperrectangles~\cite{rivest2012nonrigid}. In these cases, except in direct search, the guarantees on finding the correct global minima are very weak. In the case of direct search one requires the parameter space to be finite.

Another line of work~\cite{chui2000nonrigid,myronenko2010point,jian2011gmmreg} in point-set registration focus on using \emph{soft} or inexact correspondences to enhance the search for global optimizers. In a recent work \cite{baka2014ogmm}, the Oriented Gaussian Mixture Model (\texttt{OGMM}) method is proposed to extend the successful Gaussian mixture based nonrigid registration algorithm (\texttt{gmmreg})~\cite{jian2011gmmreg} to the 2D/3D case. These methods model the point configuration as Gaussian mixtures, and they intend to find a transformation that maximizes the overlap between these distributions. The structure of Gaussian distribution encapsulates the idea of soft correspondence and enables fast implementation. Empirically, as more fuzziness is allowed for point correspondences (larger variance for the Gaussians), the target function of optimization is smoother and hence it is less likely to find a local optimum. Nevertheless, a good initialization is still crucial for these types of algorithms due to the nonconvex nature of the cost.

Other than local optimization approaches to the registration problem, efforts have been put forth to exhaustively search for the global optimum using methods such as branch and bound ~\cite{li2007iccv}. While such techniques can be fast in practice, in principle registration problems that require exponential running time can exist. In this paper, we introduce a different approach to \emph{jointly} and \emph{globally} solve the pose and correspondence problem in 2D/3D point-set registration by relaxing the hard registration problem to convex programs. Unlike local methods, the solutions of convex programs do not depend on the initialization. We prove for certain noiseless situations the global optimizer of the surrogate convex problem coincides with the global optimizer of the original problem. Therefore through convex programming we are guaranteed to attain the global optimum in polynomial time. In the noisy regime where the solution of the convex programs closely resembles the optimizer of the original problem, the proposed method will converge to the neighborhood of the original optimizer. In these situations, it is practically advantageous to use convex optimization instead of greedy methods or global optimization that may run in exponential time. From a theoretical point of view, the convex relaxation framework can be used to characterize the polynomial time solvable cases in 2D/3D registration with unknown correspondences, through analyzing the exact recovery condition mathematically.

Before proceeding to the rest of the sections, we note that the correspondence estimation sub-problem, which is typically treated as the quadratic assignment problem, is NP-hard on its own. Therefore the solution to our proposed convex programs may not always be close to the solution of the joint pose and correspondence estimation problem. There are certainly cases for which the use greedy approaches is the only resolution, though not globally optimal. However, we hope the promising results demonstrated by the proposed novel non-iterative approach towards 2D/3D registration will attract proposals with tighter convex relaxations. The shortcomings of our approach are discussed in the concluding remarks.

\subsection{Our contributions}
The main contributions of this paper can be summarized as the following:

\textbf{Algorithm}: The original 2D/3D point-set registration of a 3D model and multiple 2D images is formulated as nonconvex MINLPs that are difficult to solve. We propose two convex relaxations for these MINLPs. The programs jointly optimize the correspondences and transformation, and the convex nature of these programs enable efficient search of global optimum regardless of initialization using standard off-the-shelf conic programming software. Furthermore, one of the convex programs $\pb$ gives solution to a variant of 2D/3D registration problem where we simultaneously estimate the point correspondences between multiple 2D images while respecting the epipolar constraint. This could be utilized in finding the corresponding image points for 3D reconstruction purposes \cite{lee2011intraoperative}.
Another natural extension of these programs is the incorporation of local descriptors of points as additional terms in the objective of both the programs. For our clinical application, we use tangency of the vessel in the local neighborhood of the point to illustrate the use of point descriptors.

\textbf{Exact recovery analysis}:
We prove exactness of the convex programs, that is, under certain conditions the proposed relaxed convex programs will give a solution to the original MINLPs. We prove that under noiseless situation the relaxed convex programs are in fact able to exactly recover the rotation and permutation matrices that match the projected points to the 3D model. Our simulations show that algorithms' global convergence results also hold with noise when the error in recovery grows nearly proportionally to the added noise. Real-data examples corroborate the theoretical
results, and suggest potential applications in coronary tree matching.

Here we outline the organization of this paper. In Section \ref{section:Notation}, we summarize the notation used in the paper. In Section \ref{section:Problem statement}, we describe the type of 2D/3D registration problems we intend to solve and present two MINLPs formulations when there are multiple images. In Section \ref{section:Convex relaxation}, we present the convexly relaxed versions of the 2D/3D registration problem in terms of tractable semidefinite programs. In Section \ref{section:Exact recovery}, we prove that achieving global optimality is possible under certain situations. In Section \ref{section:Additional features}, we mention how to incorporate additional features to the convex programs to enhance registration results, in particular in the application of coronary vessel imaging. Lastly, in Section \ref{section:Experiments}, we empirically verify the exact recovery property, and demonstrate the \emph{robustness} of the algorithm on simulated and real medical datasets for the registration of 2D coronary angiogram with the 3D model of the coronaries.

\section{Notation}
\label{section:Notation}
We use upper case letters such as $A$ to denote matrices, and lower case letters such as $t$ for vectors. We use $I_d$ to denote the identity matrix of size $d \times d$. We denote the diagonal of a matrix $A$ by diag($A$). We use $A^n$ for integer $n\geq 1$ to denote the multiplication of $A$ with itself $n$ times. We will frequently use block matrices built from smaller matrices, in particular when we deal with problem (REG2). For some block matrix $A$, we will use $A_{ij}$ to denote its $(i,j)$-th block, and $A(p,q)$ to denote its $(p,q)$-th entry. We also use $A_i$ to denote the $i$-th column of $A$. When we have an index set $s=\{s_1,\ldots,s_n\}$ where each $s_i$ is an integer, we use $A_s$ to denote the matrix $[A_{s_1},\ldots,A_{s_n}]$. We use $A \succeq 0$ to mean that $A$ is positive semidefinite. We use $\lVert x \rVert_2$ to denote the
Euclidean norm of $x \in \mathbb{R}^n$. We denote the trace of a square matrix $A$ by $\Tr(A)$. We use the following matrix norms. The Frobenius, mixed $\ell_2/\ell_1$ and entry-wise $\ell_1$ norms are defined as:
\begin{gather}
\lVert A \rVert_F = \Tr(A^T A)^{1/2}, \quad \lVert A \rVert_{2,1} = \sum_i \|A_i\|_2,\quad \mathrm{and}  \cr
 \lVert A \rVert_1 = \sum_i \sum_j \vert A(i,j)\vert.
\end{gather}
The Kronecker product between matrices $A$ and $B$ is denoted by $A \otimes B$. The all-ones vector is denoted by $\mathbf{1}$.  We use $\vert s \vert$ to denote the number of elements in a set $s$. For a set $s$ we use conv$(s)$ to denote the convex hull of $s$.
We introduce the following sets,
\begin{eqnarray*}
\Pi_d^{a\times b} \equiv \{A\in \{0,1\}^{a\times b} :  \sum_{i=1}^a A(i,j) = 1, \sum_{j=1}^b A(i,j) = 1\},\cr
\Pi_l^{a\times b} \equiv \{A\in \{0,1\}^{a\times b} :  \sum_{i=1}^a A(i,j) = 1, \sum_{j=1}^b A(i,j) \leq 1\},\cr
\Pi^{a\times b} \equiv \{A\in \{0,1\}^{a\times b} :  \sum_{i=1}^a A(i,j) \leq 1, \sum_{j=1}^b A(i,j) \leq 1\},
\end{eqnarray*}
and we frequently call them the permutation, left permutation and sub-permutation matrices.

\section{Problem statement}
\label{section:Problem statement}
A 3D centerline representation of a coronary artery tree, segmented from a preoperative CTA volume, is to be registered to $N$ fluoroscopic images. Our observed 3D model with $m$ points is described by matrix $X \in \mathbb{R}^{3 \times m}$. The projection operators between the 3D space and the $i$-th image, represented by $\Psi_{i}$, are known from the calibration of the X-ray apparatus. The projection operator maps the 3D model to a degenerate point cloud $\mathcal{I}_i\in \mathbb{R}^{3\times n_i}$ that represents $i$-th projection image. By degenerate we mean the affine rank of $\I_i$ is two. We will assume that $n_i\geq m, \ 1 \leq i \leq N$. The 2D/3D registration problem is to find an alignment matrix $R$, in the special orthogonal group $\mathbb{SO}(3)$, that matches some permutation of the observed 3D model with each of the degenerate projections.

We propose to solve the multiple images 2D/3D registration problem as:
\begingroup\makeatletter\def\f@size{9}\check@mathfonts
\def\maketag@@@#1{\hbox{\m@th\large\normalfont#1}}
\begin{equation*}
\mathrm{(REG1)}\ \ \min_{\substack{R\in \mathbb{SO}(3),\\ P_{i} \in \Pi_i^{n_i\times m} }}
\sum_{i=1}^{N} \| \Psi_i R X - \I_i P_{i} \|_{2,1},
\end{equation*}
\endgroup
where $\| \cdot \|_{2,1}$ is the mixed $\ell_2/\ell_1$ norm. We do not know a priori the correspondences between points in the 3D model and projection image (which are encoded in $P_{i}$ for each of $N$ given images), and the rotation $R$. These are found by solving (REG1). Intuitively, the minimization of the cost in (REG1) simply ensures by subsampling and permuting $\mathcal{I}_i$, the image points should as close as possible to the projections of the 3D model $X$ posed by some rotation $R$. We employ $\| \cdot \|_{2,1}$ norm in order to alleviate the costs due to the outliers. If we consider a maximum likelihood estimation framework with Gaussian type noise, a squared Frobenius norm $\| \cdot \|_F^2$ could be replaced instead.

Next we consider a variant of problem (REG1) where the correspondences of the points between the two images are available. Such correspondences could come from feature matching, or by exploiting epipolar constraints. This basically means $P_{i}$ need to be optimized dependently to preserve correspondences between the images. Let the correspondence of points between $\I_a$ and $\I_b$ be denoted by $\hat P_{ab}$, where $\hat P_{ab}(i,j)=1$ if point $i$ in $\I_a$ corresponds to point $j$ in $\I_b$ and zero otherwise. We do not require one-to-one correspondence between points such as each row or column of $\hat P_{ab}$ can have more than one nonzero entries. In the presence of such correspondence information, $R$ can be obtained as the solution to the following optimization problem:
\begingroup\makeatletter\def\f@size{9}\check@mathfonts
\def\maketag@@@#1{\hbox{\m@th\large\normalfont#1}}
\begin{equation*}
\mathrm{(REG2)}\ \min_{\substack{R\in \mathbb{SO}(3),\\ (P_{i},P_{ab}) \in \mathcal{S}}}
\sum_{i=1}^{N} \| \Psi_i R X - \I_i P_{i} \|_{2,1}
+ \sum_{a=1}^N \sum_{b>a}^N \| P_{ab} - \hat P_{ab} \|_1
\end{equation*}
\endgroup
where $P_{ab}$ is the permutation matrix that relates the points of $\I_a$ and $\I_b$ for $N(N-1)/2$ image pairs, and $\|\cdot\|_1$ is the entry-wise $\ell_1$ norm. The domain of optimization for the $\mathcal{S}$ will be explained in the next subsection. We note that solving this problem could also be useful if we are interested in estimating the correspondence between the $N(N-1)/2$ image pairs directly for reconstruction purpose.

For the remainder of this paper we only consider two images to simplify notation. It should be obvious the solution we propose in the subsequent sections can be easily extended to the multiple images case. Problems (REG1) and (REG2) have nonconvex domains which consist of integer and rotation matrices, as thus are very difficult to solve. In subsequent sections, we formally define these domains and present convex relaxation that can recover the exact solution under certain conditions.

\subsection{Domains for permutation matrices}

A more formal way to understand the problems (REG1), (REG2) is the following: Suppose there is a ground truth 3D model with $m'$ points which is described by the coordinate matrix $X_{gt}\in \mathbb{R}^{3 \times m'}$. Our observed 3D model is described by matrix $X \in \mathbb{R}^{3 \times m}$. Assume $m\leq n_1,n_2\leq m'$. In this case, we have
\begin{gather}
\I_1 = \p_1 R X_{gt} Q_1,\ \ \ \I_2 = \p_2 R X_{gt} Q_2,\label{premodel}\\
X = X_{gt} Q_3,
\end{gather}
where $R$, $\Psi_{1},\Psi_{2},\I_1,\I_2$ are matrices as introduced as before, $Q_1\in \Pi_l^{m'\times n_1}, Q_2\in \Pi_l^{m'\times n_2}$ and $Q_3 \in \Pi_l^{m'\times m}$. One can intuitively regard $Q_1, Q_2, Q_3$ as operators that generate the images and the observed model by sub-sampling and permuting the points in ground truth 3D model $X_{gt}$. For example, if the $i$th row of $Q_1$ is zero, then it means the $i$-th point of $X_{gt}$ (or more precisely, $i$-th column of $X_{gt}$) is not selected to be in $\I_1$.

Here we make an assumption, that if a point $i$ in $X_{gt}$ (the $i$-th column of $X_{gt}$) is contained in $X$, then the projections of point $i$ must correspond to some columns of $\I_1, \I_2$. Loosely speaking, it means the observed 3D model $X$ is a subset of the point clouds $\I_1,\I_2$. In this case we know $P_{1} = Q_1^T Q_3\in \Pi_l^{n_1\times m}$ and $P_{2} = Q_2^T Q_3 \in \Pi_l^{n_2\times m}$. Furthermore, if a point $i$ of $X_{gt}$  is not selected by $Q_1$ (meaning $\sum_j Q_1(i,j) = 0$), then $Q_1 Q_1^T\in \mathbb{R}^{n\times n}$ is almost an identity matrix $I_{n}$, except it is zero for the $i$th diagonal entry (similarly for $Q_2 Q_2^T$). We use these facts after multiplying the equations in (\ref{premodel}) from the right by $Q_1^T Q_3, Q_2^T Q_3$ to get
\begin{equation}
\label{model}
\Psi_1 R X  = \mathcal{I}_1 P_{1},\ \ \ \ \ \Psi_2 R X = \mathcal{I}_2 P_{2}.
\end{equation}
When there is no noise in the image, the equations in (\ref{model}) have to be satisfied. If not we turn to solve the optimization problem (REG1).

To simultaneously estimate the correspondence $P_{12}$, we use a similar construction as \cite{chen2014matchlift} and \cite{huang2013cycle}. Let $P_{12} = Q_1^T Q_2$. When having exact correspondence matrix $\hat P_{12}$ without ambiguity, we require
\begin{equation}
P_{12} = \hat P_{12}.
\end{equation}
We relate $P_{1},P_{2},P_{12}$ through a new variable
\begin{equation}
P \equiv [Q_1\  Q_2\  Q_3]^T [Q_1\  Q_2\  Q_3] \in \mathbb{R}^{(n_1+n_2+m) \times (n_1+n_2+m)}.
\end{equation}
In this context, the variables $P_{1}, P_{2}$ used in (REG1), along with $P_{12}$ are simply the off diagonal blocks of the variable $P$. Then the domain of $(P_{1},P_{2},P_{12})$ is simply
\begin{eqnarray}
\label{reg2 domain}
\mathcal{S} &=& \{(P_{1},P_{2},P_{12}): \cr
&\ &P\succeq 0, \mathrm{rank}(P)\leq n, P_{ii} = I_{n_i}, \cr
&\ &P_{1} \in \Pi^{n_1 \times m}_l, P_{2} \in \Pi^{n_2 \times m}_l, P_{12}\in \Pi^{n_1 \times n_2} \}.
\end{eqnarray}
$P_{ii}$ should be understood as the diagonal blocks of $P$. Again, we solve the optimization (REG2) when images or correspondences are noisy. This formulation ensures that the left permutation and sub-permutation matrices $P_{1},P_{2},P_{12}$ are \emph{cycle-consistent}. This basically means the relative permutation matrices $P_1,P_2,P_{12}$ all can be constructed by some underlying global ``sampling'' matrices $Q_1,Q_2,Q_3$ for each image and the 3D model (see \cite{huang2013cycle} for a detail description).

\section{Convex relaxation}
\label{section:Convex relaxation}
In this section, we propose two programs, which are the \emph{convex relaxations} of problem (REG1) and (REG2). We replace the non-convex constraints that prevents these problems from being convex to looser convex constraints. These relaxations enables the efficient search of global optimum using standard convex programming packages. To solve the convex problems we use \texttt{CVX}, a package based on interior point methods \cite{ben2001lectures} for specifying and solving convex programs \cite{cvx,gb08}. Before looking at each of the two problems, we introduce a result from \cite{saunderson2014convhull} to deal with the nonconvexity of $\mathbb{SO}(3)$ manifold. In \cite{saunderson2014convhull} the authors give a characterization of the convex hull of $\mathbb{SO}(d)$ for any dimension $d$, denoted as $\mathrm{conv}(\mathbb{SO}(d))$, in terms of positive semidefinite matrix. The $d=3$ version will be restated in the following theorem. Due to space limitation, for this theorem we locally define the (i,j)-th entry of a matrix $X$ as $X_{ij}$ instead of $X(i,j)$.

\begin{theorem}[Thm 1.3 of \cite{saunderson2014convhull}]

$\mathrm{conv}(\mathbb{SO}(3)) =\{X\in \mathbb{R}^{3\times 3} : $
\vspace{2.5 mm}
\newline
$\tiny{\Bigg[\begin{smallmatrix}
1-X_{11}-X_{22}+X_{33} & X_{13}+X_{31} & X_{12}-X_{21} & X_{23}+X_{32}\\
X_{13}+X_{31} & 1+X_{11}-X_{22}-X_{33} & X_{23}-X_{32} & X_{12}+X_{21}\\
X_{12}-X_{21} & X_{23}-X_{32} & 1+X_{11}+X_{22}+X_{33} & X_{31}-X_{13}\\
X_{23}+X_{32} & X_{12}+X_{21} & X_{31}-X_{13} & 1-X_{11}+X_{22}-X_{33}\end{smallmatrix}\Bigg]}$ is positive semidefinite $\}$.
\end{theorem}

\vspace{3 mm}

Such linear matrix inequality is convex and can be included in a semidefinite program. For both problems (REG1) and (REG2), we relax the domain of $R$, which is $\mathbb{SO}(3)$ into $\mathrm{conv}(\mathbb{SO}(3))$. However, this relaxation alone is not sufficient to turn (REG1) and (REG2) into convex problems, and we lay out the details of the fully relaxed problem in the next two subsections. We present the relaxations in the case of two images, and the generalization to more than two images should be obvious to the reader.

\subsection{Relaxing (REG1)}
We remind readers that in problem (REG1) we find a rotation to register the 3D model with the images. The points in the 3D model $X$ and $\I_1, \I_2$ should be matched after a permutation operation through the variables $P_{1}, P_{2}$. $P_{1}, P_{2}$ have their entries on the domain $\{0,1\}$. It is difficult to optimize on such discrete domain. A typical way to bypass such difficulty is by the following discrete-to-continuous relaxation,
\begin{equation}
P_{i} \in  \{0,1\}^{n_i \times m}\ \ \   \rightarrow \ \ \ P_{i} \in [0,1]^{n_i \times m}
\end{equation}
for $i = 1,2$. The relaxed matrices $P_{1}$ and $P_{2}$ can be interpreted as probabilistic map between the points in the image and the points in the 3D model. For example, the size of an entry $P_{1}(i,j)$ resembles the probability of point $i$ in $\I_1$ being corresponded to point $j$ in the 3D model. Such relaxation has been used for problems such as graph isomorphism \cite{scheinerman2011fractional}.

We now introduce the convexly relaxed problem $\pa$,
\begingroup\makeatletter\def\f@size{9}\check@mathfonts
\def\maketag@@@#1{\hbox{\m@th\large\normalfont#1}}
\begin{alignat*}{3}
&\texttt{(conreg1)}&&\ \cr
&\min_{R, P_{1}, P_{2}} &&\| \Psi_1 R X - \I_1 P_{1} \|_{2,1} + \| \Psi_2 R X- \I_2 P_{2} \|_{2,1}\cr
&\ \ s.t \ &&R\in \mathrm{conv}(\mathbb{SO}(3)),\cr
&\ &&\ P_{i} \in [0,1]^{n_i\times m},\cr
&\ &&\ P_{i} \mathbf{1} \leq 1, \mathbf{1}^T P_{i} = 1, \mathrm{for}\  \ i\in \{1,2\}.
\end{alignat*}
\endgroup

The last constraint simply follows from the definition of left permutation matrix in the notation section. In case when there are $N>2$ images, there should be $N$ matrix variables $P_1,\ldots,P_N$ for $i \in \{1,\cdots,N\}$ with constraints as indicated above.

\subsection{Relaxing (REG2)}
In problem (REG2) we include the estimation of $P_{12}$, the correspondence between image $\I_1,\I_1$, on top of pose estimation. In this case, we relax the $\mathbb{SO}(3)$ and the integrality constraints just like the case of problem (REG1). However, different from (REG1), the domain of optimization for (REG2) includes the set $\mathcal{S}$ in (\ref{reg2 domain}), which has an extra rank requirement that prevents the convexity of (REG2). A popular way of dealing with the rank constraint is by simply dropping it. Such relaxation is common in many works, for example in the seminal paper of low rank matrix completion \cite{candes2009lowrank}. We now state the convex problem $\pb$

\begingroup\makeatletter\def\f@size{9}\check@mathfonts
\def\maketag@@@#1{\hbox{\m@th\large\normalfont#1}}
\begin{alignat*}{2}
&\texttt{(conreg2)}&&\ \cr
&\min_{R, P_{1}, P_{2},P_{12}}  &&\| \Psi_1 R X - \I_1 P_{1} \|_{2,1} + \| \Psi_2 R X- \I_2 P_{2} \|_{2,1} \cr
&\ &&+ \| P_{12} - \hat P_{12} \|_1\cr
&\ \ \ s.t &&\ R\in \mathrm{conv}(\mathbb{SO}(3)),\cr
&\ &&\ P_{i} \in [0,1]^{n_i\times m},\  \mathrm{for}\  \ i\in \{1,2\}, \cr
&\ &&\ P_{12} \in [0,1]^{n_1\times n_2}, \cr
&\ &&\ P_{i} \mathbf{1} \leq 1,\  \mathbf{1}^T P_{i} = 1,\  \mathrm{for}\  \ i\in \{1,2\},\cr
&\ &&\ P_{12} \mathbf{1} \leq 1, \ \mathbf{1}^T P_{12} \leq 1, \cr
&\ &&\ \begin{bmatrix} I_{n_1} & P_{12}& P_{1}\\  P_{12}^T& I_{n_2} & P_{2}\\P_{1}^T & P_{2}^T & I_{m}\end{bmatrix} \succeq 0.\cr
\end{alignat*}
\endgroup

Again, \texttt{conreg2} inherits the equality and inequality constraints for $P_i$'s and $P_{12}$ from the the property of left permutation and sub-permutation matrices. In case when there are $N>2$ images, there should be $N$ matrix variables $P_1,\ldots,P_N$ and $N(N-1)/2$ variables $P_{ij}$. The relevant constraint that encodes cycle consistency in such case should be
\begin{eqnarray}
\begin{bmatrix} I_{n_1} & P_{12}& \cdots & P_{1N} & P_{1}\\  P_{12}^T& I_{n_2} & \cdots & P_{2N} & P_{2} \\ \vdots & \vdots & \ddots & \vdots  & \vdots \\P_{1N}^T & P_{2N}^T & \cdots & I_{n_N} & P_{N} \\ P_1^T &P_2^T& \cdots & P_N^T & I_m \end{bmatrix} \succeq 0.
\end{eqnarray}

\subsection{Connection to classical registration problem in 3D}
It is obvious that similar relaxation can be derived for the classical rigid registration problem of (\ref{classical reg}). For the sake of completeness, we state the convex formulation to the 3D registration problem in (\ref{classical reg}).
\begingroup\makeatletter\def\f@size{9}\check@mathfonts
\def\maketag@@@#1{\hbox{\m@th\large\normalfont#1}}
\begin{alignat}{3}
&\min_{R, P} &&\|  R X - Y P  \|_F^2\cr
&\ \ s.t \ &&R\in \mathrm{conv}(\mathbb{SO}(3)),\cr
&\ &&\ P \in [0,1]^{m \times m}, P \mathbf{1} = 1, \mathbf{1}^T P = 1.
\end{alignat}
\endgroup
Again, in applications where robustness is a requirement, a mixed $\ell_2/\ell_1$ can always be replaced instead. The discussion of this program is out of the scope of this paper and will be included in a future publication of the authors.

\subsection{Projection and local optimization}
Although in section \ref{section:Exact recovery} we prove that under certain situations the two convex programs give solution to (REG1) and (REG2), this will not always be the case as the domain of search has been made less restrictive. As we are concerned about posing the 3D model in this paper, we describe a strategy in this section to obtain an approximate solution from $\pa$ and $\pb$ for the rotation in general cases. Denote the optimal ``rotation'' in the domain $\mathrm{conv}(\mathbb{SO}(3))$ of both programs as $\tilde R$. As $\tilde R$ is not necessarily in $\mathbb{SO}(3)$, we need a strategy to project $\tilde R$ to a point $R^\star$ on $\mathbb{SO}(3)$. The strategy we use is simply by finding the closest orthogonal matrix to $\tilde R$ in terms of Frobenius norm, namely
\begin{equation}
R^\star = \underset{R\in \mathbb{O}(3)}{\mathrm{argmin}} \|R - \tilde R\|_F^2
\end{equation}
where $\mathbb{O}(3)$ is the orthogonal group in 3D. For completeness we provide the analytical expression for $R^\star$. Let the singular value decomposition (SVD) of $\tilde R$ be of the form
\begin{equation}
\tilde R = U \Sigma V^T.
\end{equation}
Then
\begin{equation}
R^\star = U V^T.
\end{equation}
The reader maybe curious that the special orthogonality constraint is not enforced in the rounding procedure. This is indeed unnecessary since an element $\tilde R$ in conv$(\mathbb{SO}(3))$ has positive determinant \cite{saunderson2014convhull}, hence $U V^T$ obtained from the above procedure necessarily is in $\mathbb{SO}(3)$.

As we see in the case when the solutions of $\pa$, $\pb$ are not feasible for (REG1) and (REG2), a rounding procedure is required to project $R^\star$ to a nearest point in $\mathbb{SO}(3)$. Although it is possible to prove $\|R^\star - R_0\|$ being small (as in other applications of convex relaxation, see \cite{chaudhury2013gret} for example), $R^\star$ in general is still suboptimal in terms of the cost of the original nonconvex problem (REG1) and (REG2), as the rounding procedure does not consider optimizing any cost. It is thus quite typical to use the approximate solution from convex relaxation as an initialization and search locally for a more optimal solution, for example in \cite{biswas2006snl}. In this paper, we use the \texttt{OGMM} algorithm, which searches locally by optimizing a nonconvex cost using gradient-descent based method.

\section{Exact recovery}
\label{section:Exact recovery}
In this section, we will give proof that the two relaxed convex programs exactly recover the solution to the original problems (REG1) and (REG2) under certain situations.
We will show that the solutions returned by our algorithm $\pa$ and $\pb$ lie in the domain of (REG1) and (REG2), hence showing the attainment of global optimizers of (REG1) and (REG2), as the domain of (REG1) and (REG2) are properly contained in the domain of $\pa$ and $\pb$. By this we mean \emph{exact recovery}. The proofs of exact recovery for $\pa$ and $\pb$ are almost the same. We will go through proving exact recovery for $\pa$ in detail, and just state the results for exact recovery of $\pb$.

\subsection{Our assumptions}
We summarize the assumptions for our proof.

(1) We consider the situation when $n_1 = n_2 = m = m'$, which means that all the image points $\I_1, \I_2$ come from the projection of the 3D model.

(2) We consider the noiseless situation when the equations
\begin{equation}
\I_1 \bar P_{1}= \p_1 \bar R X,\ \ \ \I_2 \bar P_{2}= \p_2 \bar R X,
\end{equation}
can be satisfied, where $\bar P_{1}, \bar P_{2} \in \Pi_d^{m\times m}$ and $\bar R \in \mathbb{SO}(3)$. We will call these variables the ground truth. Notice that under assumption (1), $\bar P_{1}, \bar P_{2}$ are permutation matrices instead of left-permutation matrices. In the case of $\pb$, we further assume
\begin{equation}
\hat P_{12} = \bar P_{12}.
\end{equation}
In order to show exact recovery, we want to show that the solution $\tilde R, \tilde P_1, \tilde P_2$ to $\pa$, is in fact $\bar R, \bar P_1, \bar P_2$. For $\pb$ we in addition show $\tilde P_{12} = \bar P_{12}$.

(3) Without lost of generality we can consider $\bar R = I_3$ and $\bar P_{1}, \bar P_{2} = I_m$. Such consideration is backed by lemma \ref{lemma:changedata}.

(4) The $m$ points are in generic positions. A set $S = \{s_1, s_2,\ldots, s_n\}$ of $n$ real numbers are called \emph{generic} if there is no polynomial $h(x_1,\ldots,x_n)$ with integer coefficients such that $h(s_1,\ldots,s_n) = 0$ \cite{connelly2005generic}. When we say the points are in \emph{generic positions}, we mean the set of $3m$ real numbers from $X$ is generic. This assumption excludes the possibilities that three points can lie on a line or four points can lie on a plane. It also implies the centroid of any subset of the points is nonzero. Furthermore, no two subsets can share the same centroid.

(5) There are at least four points, i.e. $m\geq4$.

(6) Without lost of generality, we consider $\p_1 [0\ 0\ 1]^T = \p_2 [0\ 0\ 1]^T = [0\ 0\ 1]^T$ (the directions of projection are perpendicular to $z$-axis).

We note that although we only give proof of exact recovery under these assumptions, in practice it is possible for our algorithm to recover the pose exactly when $n_1,n_2\geq m$.

\subsection{Exact recovery property of $\pa$}
Let the outcome of the optimization $\pa$ be $\tilde R, \tilde P_{1}, \tilde P_{2}$.
\begin{theorem}
\label{exactrecovery}
Under our assumptions, the solution $\tilde R, \tilde P_{1}, \tilde P_{2}$ to $\pa$ is unique, and $\tilde R = \bar R, \tilde P_{1} = \bar P_{1}, \tilde P_{2} = \bar P_{2}$.
\end{theorem}

When we have exact recovery, $R^\star$ obtained from the projection of $\tilde R$ to $\mathbb{SO}(3)$ is simply $\tilde R$, as $\tilde R$ is already in $\mathbb{SO}(3)$. To prove the theorem, we first state a couple lemmas.
\begin{lemma}
\label{lemma:changedata}
If we make the following change of data
\begin{equation}
\label{changedata}
X \rightarrow X',\ \ \ \I_1 \rightarrow \I_1',\ \ \ \I_2 \rightarrow \I_2',
\end{equation}
where
\begin{equation}
X' = \bar R X,\ \ \ \I_1' = \I_1 \bar P_{1},\ \ \ \I_2' =  \I_2 \bar P_{2},
\end{equation}
then the solutions of $\pa$ are changed according to the invertible maps
\begin{equation}
\label{changevariable}
\tilde R \rightarrow \tilde R \bar R^T,\ \ \ \tilde P_{1} \rightarrow \bar P_{1}^T \tilde P_{1},\ \ \ \tilde P_{2} \rightarrow \bar P_{2}^T \tilde P_{2}.
\end{equation}
\end{lemma}
\begin{proof}
\begin{eqnarray}
&\ &\|\I_1 \tilde P_{1}- \p_1 \tilde R X \|_{2,1} + \|\I_2 \tilde P_{2} - \p_2 \tilde R X \|_{2,1}\cr
&=& \|\I_1 \bar P_{1} \bar P_{1}^T \tilde P_{1} -  \p_1 \tilde R \bar R^T \bar R X\|_{2,1}+\cr
&\ &\|\I_2 \bar P_{2} \bar P_{2}^T \tilde P_{2} -  \p_2 \tilde R \bar R^T \bar R X\|_{2,1}\cr
&=& \|\I_1' \bar P_{1}^T \tilde P_{1} -  \p_1 \tilde R \bar R^T X'\|_{2,1}+\cr
&\ &\|\I_2' \bar P_{2}^T \tilde P_{2} -  \p_2 \tilde R \bar R^T X'\|_{2,1}.
\end{eqnarray}
The last second equality is due to the fact that $\bar P_{1} \bar P_{1}^T = I_m, \bar P_{2} \bar P_{2}^T = I_m, \bar R^T \bar R = I_3$, since permutation matrices are necessarily orthogonal matrices. Therefore, a solution after the change of data in (\ref{changedata}) is $\bar P_{1}^T \tilde P_{1}$, $\bar P_{2}^T \tilde P_{2}$ and $\tilde R \bar R^T$.
\end{proof}
Since the map presented in (\ref{changevariable}) is invertible, we know the solutions to $\pa$ with 3D model $X$ and image $\I_1, \I_2$ are in one-to-one correspondence with solutions to $\pa$ with 3D model $\bar R X$ and image $\I_1 \bar P_{1}, \I_2 \bar P_{2}$. Therefore without loss of generality, using the change of variable introduced in (\ref{changedata}), we can let the 3D model being $X' = \bar R X$, the images being $\I_1' = \I_1 \bar P_{1} = \Psi_1 \bar R X  = \Psi_1 X' $ and $\I_2' = \Psi_2 X'$. We thus consider solving the problem $\pa$ in the following form
\begingroup\makeatletter\def\f@size{9}\check@mathfonts
\def\maketag@@@#1{\hbox{\m@th\large\normalfont#1}}
\begin{alignat}{3}
&\min_{R, P_{1}, P_{2}} &&\| \p_1 R X' - \p_1 X' P_{1} \|_{2,1} + \| \Psi_2 R X'- \p_2 X' P_{2} \|_{2,1}\cr
&\ \ s.t \ &&R\in \mathrm{conv}(\mathbb{SO}(3)),\cr
&\ &&\ P_{i} \in [0,1]^{n_i\times m},\cr
&\ &&\ P_{i} \mathbf{1} \leq 1, \mathbf{1}^T P_{i} = 1, \mathrm{for}\  \ i\in \{1,2\}.
\end{alignat}
\endgroup

In this case, instead of showing $\tilde R = \bar R, \tilde P_{1} = \bar P_{1}, \tilde P_{2} = \bar P_{2}$, it suffices to show $\tilde R = I_3$ and $\tilde P_{1}, \tilde P_{2} = I_m$ being the unique solution to $\pa$.

\begin{lemma}
\label{lemma:consistency}
\begin{eqnarray}
\p_1 \tilde R X' &=& \p_1 X' \tilde P_{1} \label{I1 consistency}\\
\p_2 \tilde R X' &=& \p_2 X' \tilde P_{2} \label{I2 consistency}
\end{eqnarray}
and
\begin{eqnarray}
\p_1 \tilde R^n X' &=& \p_1 X' \tilde P^n_{1}\label{I1 power consistency}\\
\p_2 \tilde R^n X' &=& \p_2 X' \tilde P^n_{2}\label{I2 power consistency}
\end{eqnarray}
for any $n\geq 1$.
\end{lemma}
\begin{proof}
See appendix \ref{section:consistency}
\end{proof}

\begin{lemma}
\label{lemma:ergodicity}
$\tilde P_{1}$, $\tilde P_{2}$ are doubly stochastic matrices and it suffices to consider $\|\tilde P_{1} - I_m \|_1< 1$ and $\|\tilde P_{2} - I_m \|_1< 1$. Then
\begin{equation}
\lim_{n\to\infty}   \tilde P_{1}^n  = \lim_{n\to\infty}   \tilde P_{2}^n =  A,
\end{equation}
where $A$ is partitioned by index sets $a_1,\ldots,a_k$ and
\begin{equation}
A(a_i, a_i) = (1/\vert a_i \vert) \mathbf{1}\mathbf{1}^T,\ \ \ A(a_i, a_j) = 0\ \ \mathrm{for}\ i\neq j.
\end{equation}
\end{lemma}
\begin{proof}
See appendix \ref{section:ergodicity}
\end{proof}

We are now ready to prove theorem \ref{exactrecovery}.
\begin{proof}
We first want to show $\tilde R = I_3$. Taking limits of (\ref{I1 power consistency}), (\ref{I2 power consistency}) and using lemma \ref{lemma:ergodicity}, we get
\begin{equation}
\p_1 B X' = \p_1 X' A,\ \ \p_2 B X' = \p_2 X' A,
\end{equation}
where $B = \lim_{n\to\infty} \tilde R^n$ and $A = \lim_{n\to\infty} \tilde P_{1}^n = \lim_{n\to\infty} \tilde P_{2}^n$ . We saw in lemma \ref{lemma:ergodicity} $A$ has multiple irreducible components, each being an averaging operator for the points of relevant indices. The equations lead to
\begin{equation}
\label{stable consistency}
B X' = X' A.
\end{equation}
Since $B = \lim_{n\to\infty}\tilde R^n =  \tilde R \lim_{n\to\infty}\tilde R^{n-1} = \tilde R B$, multiplying (\ref{stable consistency}) from the left by $\tilde R$ we get
\begin{equation}
\label{eigenvalue R}
\tilde R X' A = X' A
\end{equation}

Now we show $\tilde R = I_3$ for each of the following cases.

Case 1: Suppose $A$ has three or more than three irreducible components. Let the centroid of any three components be $[c_1, c_2, c_3] \in \mathbb{R}^{3\times 3}$. By equation (\ref{eigenvalue R}) $\tilde R [c_1, c_2, c_3] = [c_1, c_2, c_3]$. By the assumption that columns of $X'$ are generic, $[c_1, c_2, c_3]$ has full rank hence invertible. Thus $\tilde R = I_3$.

Case 2: Suppose $A$ has two irreducible components. This means $\tilde R [c_1,c_2] = [c_1,c_2]$ where $c_1$ and $c_2$ are independent by the assumption that points are generic. Let $\hat c_1$ and $\hat c_2$ denote the normalized version of $c_1, c_2$. $\tilde R = \p_{c_1 c_2} + \alpha \hat c_1\wedge \hat c_2$, where $\p_{c_1 c_2}$ denotes the projection operator to the plane containing $c_1, c_2$, $\wedge$ denotes the vector cross product and $\alpha$ some scalar constant with $\vert \alpha \vert \leq 1$. We know  $\vert \alpha \vert \leq 1$ since the largest singular value of $\tilde R$ is less than 1, as
\begin{equation}
\mathrm{conv}(\mathbb{SO}(3))\subset \mathrm{conv}(\mathbb{O}(3)) = \{O\in \mathbb{R}^{3\times3} \vert O^T O \preceq I_3 \}
\end{equation}
where $\mathrm{conv}(\mathbb{O}(3))$ is the convex hull of the orthogonal group in 3D. Such characterization of the convex hull of the orthogonal group can be found in \cite{sanyal2011orbitopes}. If $\vert \alpha \vert = 1$, then we indeed have proven $\tilde R = I_3$. If $\vert \alpha \vert < 1$, we will arrive at a contradiction. In this situation $B = \lim_{n\to\infty}\tilde R^n =   \p_{c_1 c_2}$. By the assumption that $m\geq 4$ and equation (\ref{stable consistency}), there exists four points with coordinates $Y = [Y_1, Y_2, Y_3, Y_4] \in \mathbb{R} ^{3 \times 4}$ such that either
\begin{equation}
\p_{c_1 c_2} Y = [c_1, c_1, c_1, c_2],\ \mathrm{or}\ \p_{c_1 c_2} Y = [c_1, c_1, c_2, c_2].
\end{equation}
This means either $Y_1, Y_2, Y_3$ form a line, or $Y_1, Y_2, Y_3, Y_4$ form a plane and both cases violate the assumption that points are in generic positions.

Case 3: Suppose $A$ has a single irreducible component. In this case $B X' = [c_1,\ldots,c_1]$, where $c_1$ is the center of the point cloud. It is easy to show that $B = \p_{c_1}$, where $\p_{c_1}$ is the projection onto the line spanned by $c_1$. However, $\p_{c_1} X' = [c_1,\ldots,c_1]$ implies columns of $X'$ forms a plane with $c_1$ being the normal vector. Again this violates the assumption of generic positions.

Now since $\tilde R = I_3$, from equations (\ref{I1 consistency}) and (\ref{I2 consistency}) we have
\begin{equation}
[0\ 0\ 1]  X'  = [0\ 0\ 1] X' \tilde P_{1},\ \ \ [0\ 0\ 1]  X'  = [0\ 0\ 1] X' \tilde P_{2}.
\end{equation}
Again from the generic positions assumption, $[0\ 0\ 1] X' \in \mathbb{R}^m$ is a vector with distinct entries. It is known that in this situation, the two doubly stochastic matrices $\tilde P_{1}, \tilde P_{2} = I_m$ (see corollary 8 of \cite{brualdi1984doubly}).
\end{proof}

\subsection{Exact recovery property of $\pb$}
Let the solution of $\pb$ be $\tilde R$, $\tilde P_{1}$, $\tilde P_{2}$, $\tilde P_{12}$.
\begin{theorem}
\label{exactrecovery2}
Under our assumptions, the solution $\tilde R, \tilde P_{1}, \tilde P_{2}, \tilde P_{2}$ to $\pb$ is unique, and $\tilde R = \bar R, \tilde P_{1} = \bar P_{1}, \tilde P_{2} = \bar P_{2}, \tilde P_{12} = \bar P_{12}$.
\end{theorem}
\begin{proof}
The recovery of $\bar R, \bar P_{1}, P_{2}$ can be proven in the exact manner as in previous subsection. If $\hat P_{12} = \bar P_{12}$, $\| \tilde P_{12} - \hat P_{12} \| > 0$  unless $\tilde P_{12} = \bar P_{12}$. We therefore conclude our proof.
\end{proof}

Although it does not seem like the additional constraints in $\pb$ helps our proof at this point, in practice in the case when $n_1,n_2\geq n_3$, $\pb$ demonstrates slight advantage over $\pa$ in simulation in terms of stability (Fig. \ref{fig:fullsim}b) and also in real data (Fig. \ref{fig:convex real}).

\section{Clinical Application and Additional Features}\label{clinical background}
Though the problem of 2D/3D registration appears in several medical imaging applications, this paper is motivated by the clinical setting in which two or more fluoroscopic images of the coronaries are used for intervention guidance. In a procedure called percutaneous coronary intervention (PCI), commonly referred to as angioplasty, a cardiac interventionalist introduces a thin, flexible tube with a device on the end via arterial access (usually femoral) into the patient's coronary arteries which may be blocked by calcification.  To guide and position the device accurately, the cardiac interventionalist  uses X-ray fluoroscopy with intermittent contrast injections (angiography). The resultant fluoroscopic images contain projected outlines of the coronary vessels and blockages if any. However, these intra-operative fluoroscopic images do not have the performance characteristic of pre-operative images such as Computed Tomography Angiography (CTA) with blood pool contrast injection that enables the occlusions causing calcifications to be clearly distinguished. The benefit of using pre-operative Coronary CTA images are well documented~\cite{hidehiko2013,rolf2013jcvi,cheung2010,qu2014} and are becoming standard of care. Coronary CTA may also allow better determination of calcification, lesion length, stump morphology, definition of post-CABG anatomy, and presence of side branches compared to fluoroscopy or angiography alone~\cite{rolf2013jcvi}.

To aid the cardiac interventionalist performing the procedure on the coronaries, a strategy to provide additional information by fusing pre-operative images with intra-operative fluoroscopic has been developed. However, finding a transformation to align pre-operative and intra-operative images
is challenging since finding inter modal correspondences is a nontrivial problem. Typically to achieve this multi-modality fusion, vessel centerlines segmented from both modalities are used as landmarks for inter modal correspondences. A 3D model of the vessels extracted from pre-procedure CTA is aligned with the 2D fluoroscopic views, and overlaid on top of the live fluoroscopy images, thereby augmenting the interventional images.

In this paper, we process the CTA images using the segmentation algorithm presented in~\cite{zheng2013robust} and fluoroscopic images using~\cite{deschamps2001fast,sironi2014multiscale} to get the respective point-sets. Thus, the clinical alignment transformation between pre-operative and intra-operative images is reduced to \texttt{(REG1)} or \texttt{(REG2)}.

Though this paper was motivated by the clinical scenario of CTO, our algorithms are generalized to handle other clinical situations where natural structures can be obtained by image segmentation methods. For instance, vessels from brain, liver, kidney or peripheral images can be used by processing them with relevant segmentation algorithms. As illustrated in Section~\ref{S:bunny}, the algorithm is not restricted to point-sets from vessel like structures only. Other structures, notably surfaces, extracted by means of segmentation algorithm can be used to obtain desired point sets required by methods \texttt{(REG1)} or \texttt{(REG2)}. Typically, the structure representations extracted from segmentation methods need to be re-sampled uniformly to result in respective point-sets.

\subsection{Incorporating additional features}
\label{section:Additional features}
In the context of finding an alignment transformation between two point clouds, point descriptors maybe valuable by favoring transformation that matches the descriptors as well as the coordinates. Our proposed formulations easily allow the incorporation of descriptors by adding additional terms in the cost of (REG1), (REG2) that encourage matching of transformation invariant features, or features transforming according to rotation. To match transformation invariant features, we simply add the following terms:
\begin{equation}
\|X^F - \I_1^F P_{1} \|_{2,1} + \|X^F - \I_2^F P_{2} \|_{2,1},
\end{equation}
to our cost in (REG1) and (REG2). $X^F \in \mathbb{R}^{d\times m}$, $\I_1^F \in \mathbb{R}^{d\times n_1}$, and $\I_2^F \in \mathbb{R}^{d\times n_2}$ denote some $d$ dimensional feature vectors for the points in 3D model and the images, respectively. To deal with features that transform with rotation, such as intensity gradient, the following terms can be added to (REG1) and (REG2):
\begin{equation}
\|\p_1 R X^F - \I_1^F P_{1} \|_{2,1} + \|\p_2 R X^F - \I_2^F P_{2} \|_{2,1}.
\end{equation}

\subsection{Constructing features for coronary vessel}
To incorporate features for the specific application of registering coronary vessel, we will leverage the idea of patch \cite{buades2005nlm}. We note that the typical features for vessels, such as the tangent of the vessel, are in general extracted from the local neighborhood of the particular point \cite{lowe1999sift}.
Therefore, we simply combine the coordinates of neighborhood points around each point, which we denote as patch, as the descriptor for each point. For the coordinate of point $i$, we concatenate it with the coordinate of subsequent points $i+1, i+2, \ldots, i+n_p-1$ to form a patch of size $n_p$. We denote the result of such concatenation as $X_i^P$, ${\I_1^P}_i$ and ${\I_2^P}_i$, and we store it as the $i$-th column of the matrices $X^P \in \mathbb{R}^{3n_p  \times m}$ and $\I_1 \in \mathbb{R}^{3n_p  \times n_1}$, $\I_2 \in \mathbb{R}^{3n_p  \times n_2}$. Now instead of registering point, our goal is then to register these patches of the 3D model to the images. The cost in (REG1) when considering matching patches become
\begin{equation}
\| (I_{n_p} \otimes \Psi_1 R) X^P - \I_1^P P_{1} \|_{2,1} + \| (I_{n_p} \otimes \Psi_2 R) X^P- \I_2^P P_{2} \|_{2,1},
\end{equation}
and for (REG2) the cost is
\begin{gather}
\| (I_{n_p} \otimes \Psi_1 R) X^P - \I_1^P P_{1} \|_{2,1} + \cr
\| (I_{n_p} \otimes \Psi_2 R) X^P- \I_2^P P_{2} \|_{2,1} + \| P_{12} - \hat P_{12} \|_1.
\end{gather}
In our experiments, we found $n_p = 3$ suffices to improve the solution without introducing too many variables into the optimization.
A practical issue is that one needs to take care of the sampling density of the points on the centerline. If the point spacing is very different between $X$ and $\I_1$, $\I_2$, the patches cannot be well matched. To deal with this issue, we first project the randomly posed 3D model using $\p_1$ and $\p_2$, calculate the average spacing between two subsequent points in the projection, and sample the image points $\I_1$, $\I_2$ according to this spacing. 
\section{Experiments}
\label{section:Experiments}

\begin{figure}
\centering
\includegraphics[width=0.45\textwidth]{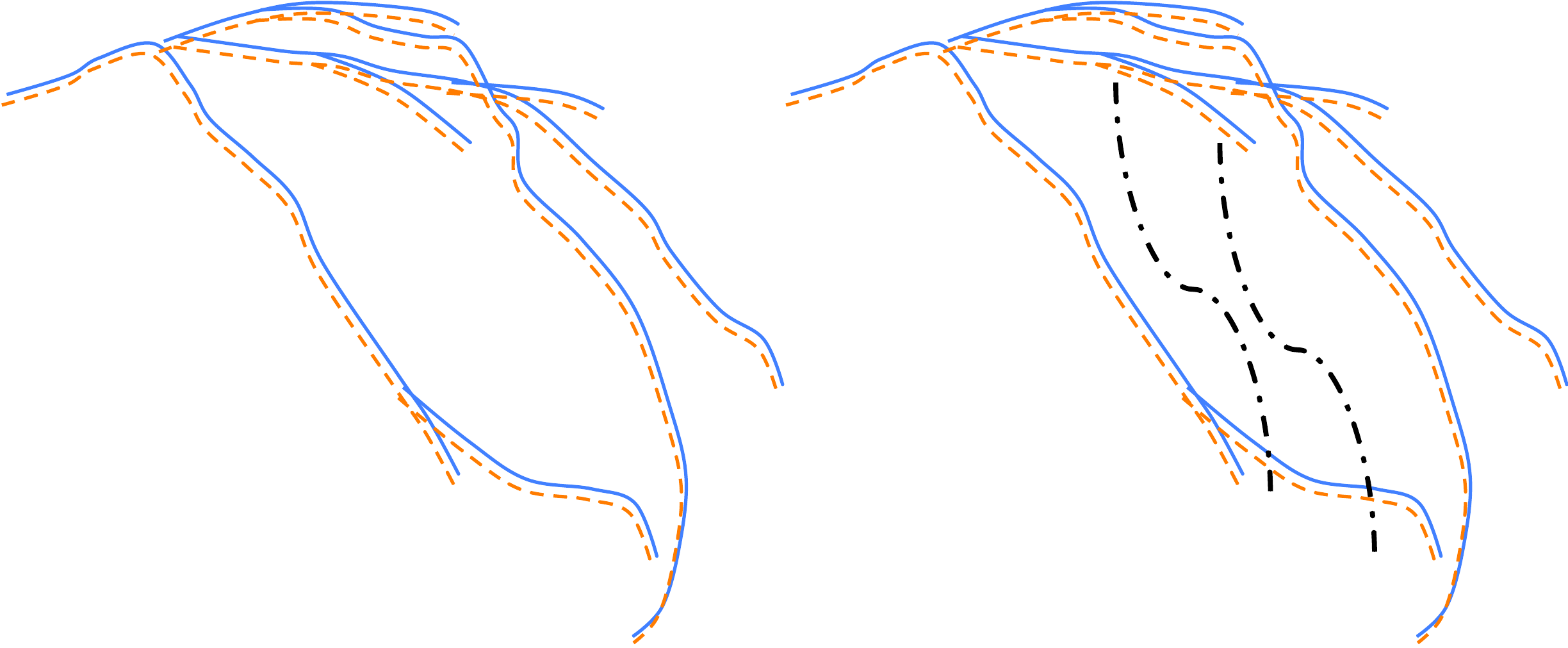}
\caption{An illustration of synthetic data used for simulations. Solid blue lines represent the centerlines on 2D and dashed red lines represent the projection of 3D model on the image plane. The left figure illustrates the case when the 3D model is a proper subset.}\label{fig:full partial}
\end{figure}

\begin{figure}
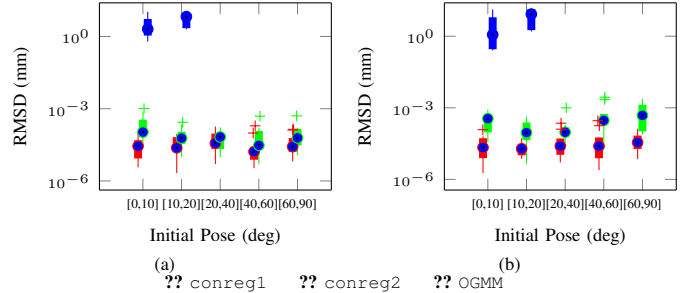

  \centering
  \setlength\fheight{2.5cm}
    \setlength\fwidth{0.17\textwidth}
\begin{minipage}{0.48\columnwidth}{\input{figures/simdatafull_angles.tex}}
~\\[-4ex] \centering{\scriptsize(a)}
\end{minipage}
\begin{minipage}{.01\columnwidth}
\hfill
\end{minipage}
\begin{minipage}{0.48\columnwidth}{\input{figures/simdatapartial_angles.tex}}
~\\[-4ex] \centering{\scriptsize(b)}
\end{minipage}
\\
\scriptsize{
\begin{tabular}{c c c}
 \ref{pgfplots:lreg1} \texttt{conreg1} & \ref{pgfplots:lreg2} \texttt{conreg2} & \ref{pgfplots:ogmm} \texttt{OGMM} \\
\end{tabular}
}
    \caption{Simulation results of the two convex programs and \texttt{OGMM} (a) when the images are the projection of the complete 3D model. and (b) when the images contain the projections of the 3D model as a proper subset. Here we only plot the cases when \texttt{OGMM} finds a solution within 25 mm RMSD. If for a particular rotation range there is no boxplot it means no solution has less than 25 mm RMSD for that range.}
    \label{fig:fullsimnoiseless}
\end{figure}

In this section we evaluate the performance of our algorithm using synthetic data of the heart vessels, and real CT data of six different patients. In particular, we demonstrate through simulation that regardless of the initial pose of the 3D model, our algorithm exactly recovers the ground truth pose when there is no noise. In the presence of image noise, our algorithm is able to return a solution close to the ground truth. These observations are also verified when working with real data. For these experiments we solve \texttt{conreg1} and \texttt{conreg2} using the convex programming package \texttt{CVX} in MATLAB using interior point method. In general the time complexity of interior-point-method for conic programming is $O(N^{3.5})$ where $N$ denotes the total number of variables \cite{mosekcomplexity}. The main source of complexity in our program comes from the stochastic matrices $P_1$ and $P_2$. Assuming there are $m$ points, the time complexity is then $O(N^{3.5})\sim O(m^7)$. Since the interior-point-method is a second order method requiring the storage of Hessian, the space complexity is $O(N^2)\sim O(m^4)$. While this looks intimidating, as reported in the upcoming sections it is quite fast in practice due to the fast implementation of the solvers in \texttt{CVX}. To measure the quality of registration, we use the root-mean-square-distance
\begingroup\makeatletter\def\f@size{9}\check@mathfonts
\def\maketag@@@#1{\hbox{\m@th\large\normalfont#1}}
\begin{equation}
\label{RMSD}
\mathrm{RMSD} = \frac{1}{2\sqrt{m}} (\|\p_1 R^\star X - \p_1 \bar R X\|_F + \|\p_2 R^\star X - \p_2 \bar R X\|_F).
\end{equation}
\endgroup
As a reminder, $R^\star$ is the solution after rounding from $\pa,\pb$, and $\bar R$ is the ground truth rotation. $\bar R$ is known in the case of synthetic data, and for the case of real data it is manually given by medical experts. The unit of RMSD will be in millimeter.

\subsection{Synthetic data}

We use synthetic data to demonstrate the ability of our algorithm to exactly recover the pose of the 3D model when there is no noise in the image. We also add bounded noise to each point in the image in the following way
\begin{gather}
\label{noisemodel}
{\I_1}_i = \p_1 (R X_i + {\varepsilon_1}_i),\ \ \  {\I_2}_j = \p_2 (R X_j + {\varepsilon_2}_j),\cr
i\in 1,\ldots,n_1,\ j\in 1,\ldots,n_2,
\end{gather}
where ${\varepsilon_1}_i$, ${\varepsilon_2}_j \sim \mathcal{U}[-\varepsilon,\varepsilon]^3 $ and $\mathcal{U}[-\varepsilon,\varepsilon]^3$  is the uniform distribution in the cube $[-\varepsilon,\varepsilon]^3$.

In order to run \texttt{conreg2}, we simply get $\hat P_{12}$ from the epipolar constraints. For a point in image 1, denoted as $x\in \mathbb{R}^3$ in homogeneous pixel coordinate, the epipolar line $l$ in image 2 can be computed as $l=F x$, where $F\in \mathbb{R}^{3\times 3}$ denotes the fundamental matrix. If point $x'$ in image 2 satisfies $\vert l^T x' \vert = \vert x^T F x' \vert \leq \delta$ where $\delta$ is some pre-defined threshold, we then set the entry in $\hat P_{12}$ corresponding to the points $x$ and $x'$ to be 1. We pick $\delta$ to be 2 to 3 times the average spacing between the neighboring points.

We first start with the Stanford bunny dataset to illustrate the exact recovery of pose from \texttt{conreg1} and \texttt{conreg2}. We first project the bunny in $x$ and $y$-directions to obtain $\I_1$ and $\I_2$, respectively. We then pose the model bunny $X$ with arbitrary rotation and register it to its projections. In Fig. \ref{fig:bunny}a, we show the registration results of \texttt{conreg1} in one viewing direction with $m=473$. When there is no noise, the exact recovery of pose is verified in Fig.~\ref{fig:bunny}. While we do not give guarantees on the approximate recovery of pose when there is noise, we also show in Fig. \ref{fig:bunny}b that our proposed methods are stable when perturbed by noise. We note that for Fig. \ref{fig:bunny}b, we sparsify the number of points to $m=60$ and the average distance between neighboring points is 0.04. Although \texttt{conreg1} can solve registration problem with $m=473$ in less than 10 seconds, the computational cost of \texttt{conreg2} is high and we have to limit the number of points.
\begin{figure}
\centering
\subfloat[Bunny $\varepsilon = 0$]{\includegraphics[width = 0.18\textwidth]{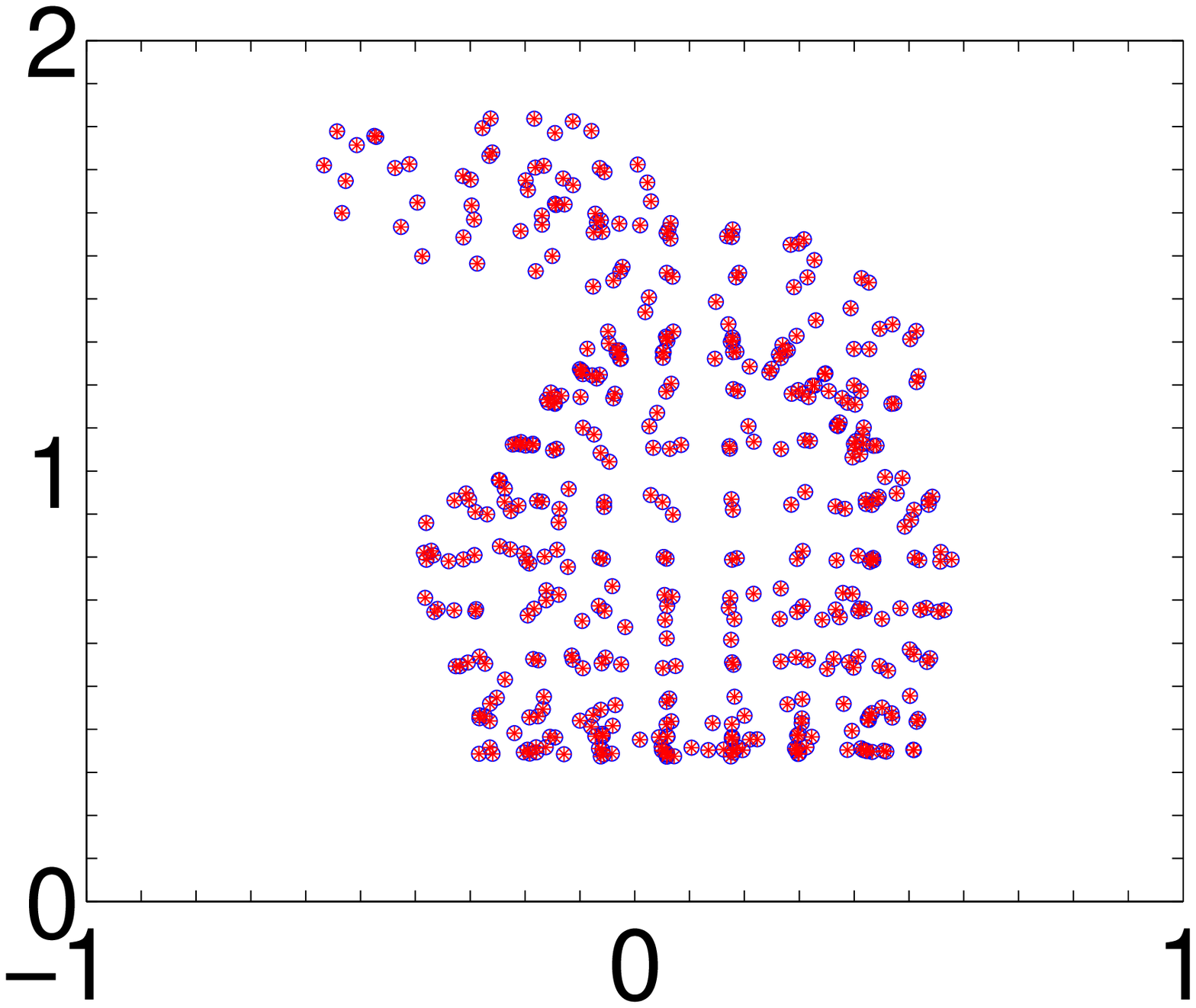}}
\subfloat[Stability]{\includegraphics[width = 0.2\textwidth]{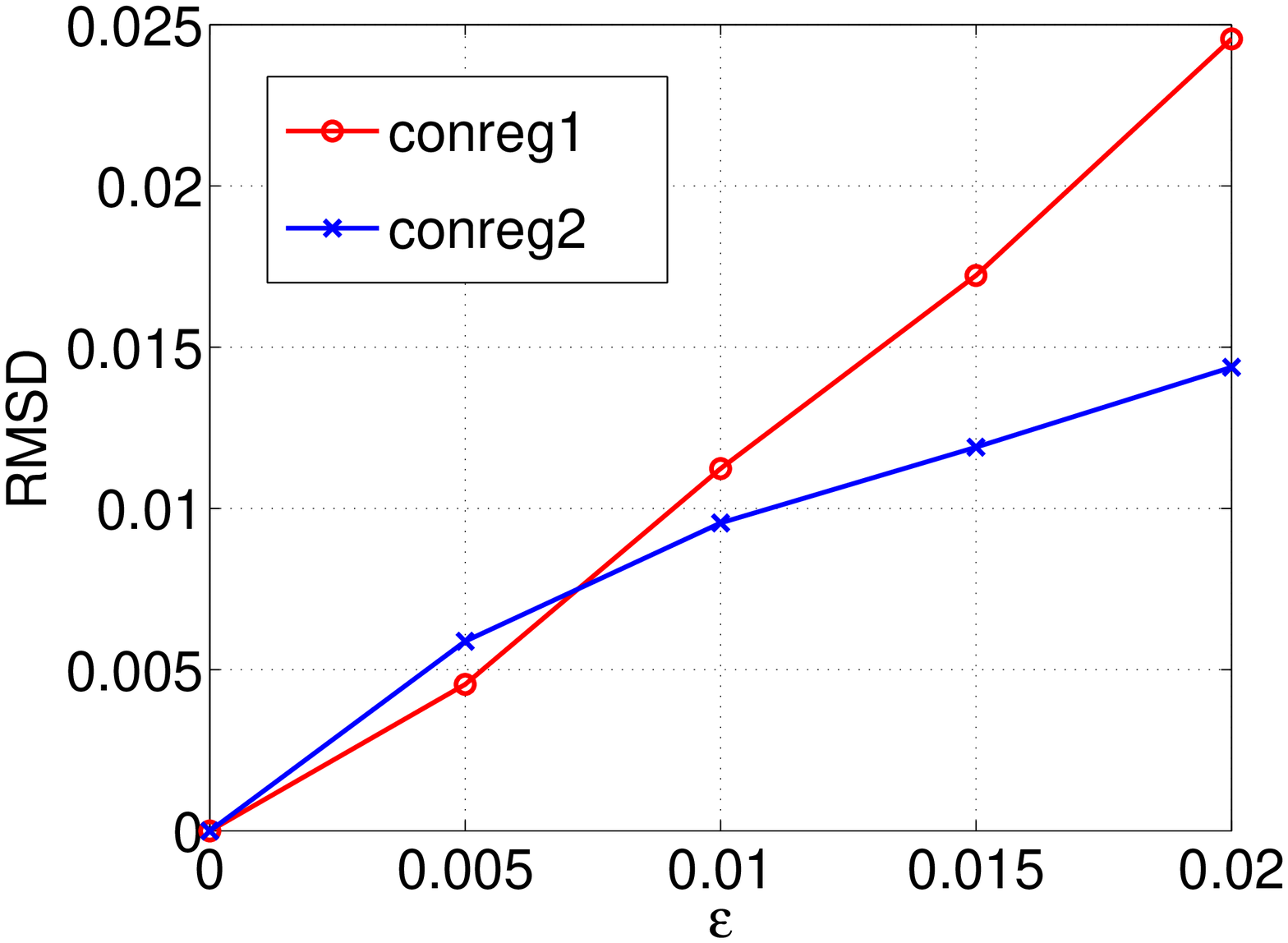}}
\caption{(a) A simulation using the Stanford bunny dataset when there is no noise. The red points are taken from $\I_1$, and the blue points are taken from the model $X$ posed by the solution from \texttt{conreg1}. The pose is exactly recovered when there is no noise. (b) RMSD v.s. noise level for the bunny dataset. Each data point is averaged over 20 noise and pose realizations. The average spacing between the points is about 0.04. }\label{fig:bunny}
\end{figure}

For the registration of coronary vessels, we perform simulations for two different cases (Fig. \ref{fig:full partial}), namely, when the 2D images are exactly the projections of the 3D model and when there are extra centerlines in the 2D images. In the latter case, the 3D model is a proper subset. 
In our simulations, we first rotate the 3D model of coronary vessels into some arbitrary posture and project it in two fixed orthogonal directions. We tested our algorithm systematically when rotation $R$ in (\ref{noisemodel}) is within each of the ranges $[0^\circ,10^\circ],[10^\circ, 20^\circ],[20^\circ, 40^\circ],[40^\circ, 60^\circ]$, and $[60^\circ, 90^\circ]$, using 10 different noise realizations in each case.

\begin{figure}
  \centering
  \setlength\fheight{2.5cm}
    \setlength\fwidth{0.18\textwidth}
\begin{minipage}{0.48\columnwidth}{
%
%
%
\begin{tikzpicture}

\begin{axis}[%
width=\fwidth,
height=0.996313\fheight,
at={(0\fwidth,0\fheight)},
scale only axis,
unbounded coords=jump,
separate axis lines,
every outer x axis line/.append style={black},
every x tick label/.append style={font=\color{black}},
xmin=0.875,
xmax=5.375,
xmode=log,
xminorticks=true,
xtick={1.125,2.125,3.125,4.125,5.125},
xticklabels={{$\infty$},{21.6},{15.6},{12.1},{9.6}},
xlabel={Noise (dB)},
every outer y axis line/.append style={black},
every y tick label/.append style={font=\color{black}},
ymin=0,
ymax=12,
ylabel={RMSD (mm)},
legend columns=-1,
legend to name=simdataconvex2_legend,
label style={font=\scriptsize}, legend style={font=\footnotesize}, every x tick label/.append style={font=\tiny}, every y tick label/.append style={font=\tiny}
]
\addplot [color=red,solid,forget plot]
  table[row sep=crcr]{%
1	0\\
1	0\\
};
\addplot [color=green,solid,forget plot]
  table[row sep=crcr]{%
1.25	0\\
1.25	0\\
};
\addplot [color=red,solid,forget plot]
  table[row sep=crcr]{%
2	1.45844229123501\\
2	2.63059638624397\\
};
\addplot [color=green,solid,forget plot]
  table[row sep=crcr]{%
2.25	1.47988758612838\\
2.25	2.49834089422492\\
};
\addplot [color=red,solid,forget plot]
  table[row sep=crcr]{%
3	2.98160679223802\\
3	5.07896765000554\\
};
\addplot [color=green,solid,forget plot]
  table[row sep=crcr]{%
3.25	3.01138609285593\\
3.25	4.87531650165371\\
};
\addplot [color=red,solid,forget plot]
  table[row sep=crcr]{%
4	4.99899219690851\\
4	7.89142555173809\\
};
\addplot [color=green,solid,forget plot]
  table[row sep=crcr]{%
4.25	4.94617249539616\\
4.25	7.57398089154955\\
};
\addplot [color=red,solid,forget plot]
  table[row sep=crcr]{%
5	6.43525708645747\\
5	10.8116582762301\\
};
\addplot [color=green,solid,forget plot]
  table[row sep=crcr]{%
5.25	6.4226962087229\\
5.25	10.8116582762301\\
};
\addplot [color=red,solid,line width=3.0pt,forget plot]
  table[row sep=crcr]{%
1	0\\
1	0\\
};
\addplot [color=green,solid,line width=3.0pt,forget plot]
  table[row sep=crcr]{%
1.25	0\\
1.25	0\\
};
\addplot [color=red,solid,line width=3.0pt,forget plot]
  table[row sep=crcr]{%
2	1.71155242429631\\
2	2.08627141294257\\
};
\addplot [color=green,solid,line width=3.0pt,forget plot]
  table[row sep=crcr]{%
2.25	1.64704032406245\\
2.25	2.0755460192929\\
};
\addplot [color=red,solid,line width=3.0pt,forget plot]
  table[row sep=crcr]{%
3	3.62560319279253\\
3	4.64056654255863\\
};
\addplot [color=green,solid,line width=3.0pt,forget plot]
  table[row sep=crcr]{%
3.25	3.43244219781927\\
3.25	4.31585139134432\\
};
\addplot [color=red,solid,line width=3.0pt,forget plot]
  table[row sep=crcr]{%
4	5.72377253764527\\
4	6.83918539330669\\
};
\addplot [color=green,solid,line width=3.0pt,forget plot]
  table[row sep=crcr]{%
4.25	5.51620659033776\\
4.25	6.52134013357534\\
};
\addplot [color=red,solid,line width=3.0pt]
  table[row sep=crcr]{%
5	7.59673811140029\\
5	9.18250636127816\\
};
\addlegendentry{\texttt{conreg1}};
\label{pgfplots:simconreg1}

\addplot [color=green,solid,line width=3.0pt]
  table[row sep=crcr]{%
5.25	7.50803458226248\\
5.25	9.10179386302299\\
};
\addlegendentry{\texttt{conreg2}};
\label{pgfplots:simconreg2}

\addplot [color=blue,only marks,mark=*,mark options={solid,fill=blue,draw=red},forget plot]
  table[row sep=crcr]{%
1	0\\
};
\addplot [color=blue,only marks,mark=*,mark options={solid,fill=blue,draw=green},forget plot]
  table[row sep=crcr]{%
1.25	0\\
};
\addplot [color=blue,only marks,mark=*,mark options={solid,fill=blue,draw=red},forget plot]
  table[row sep=crcr]{%
2	1.98521820039279\\
};
\addplot [color=blue,only marks,mark=*,mark options={solid,fill=blue,draw=green},forget plot]
  table[row sep=crcr]{%
2.25	1.8557762561928\\
};
\addplot [color=blue,only marks,mark=*,mark options={solid,fill=blue,draw=red},forget plot]
  table[row sep=crcr]{%
3	4.09583548071032\\
};
\addplot [color=blue,only marks,mark=*,mark options={solid,fill=blue,draw=green},forget plot]
  table[row sep=crcr]{%
3.25	3.8909599104109\\
};
\addplot [color=blue,only marks,mark=*,mark options={solid,fill=blue,draw=red},forget plot]
  table[row sep=crcr]{%
4	6.30279414388284\\
};
\addplot [color=blue,only marks,mark=*,mark options={solid,fill=blue,draw=green},forget plot]
  table[row sep=crcr]{%
4.25	6.15815458220558\\
};
\addplot [color=blue,only marks,mark=*,mark options={solid,fill=blue,draw=red},forget plot]
  table[row sep=crcr]{%
5	8.48053899114163\\
};
\addplot [color=blue,only marks,mark=*,mark options={solid,fill=blue,draw=green},forget plot]
  table[row sep=crcr]{%
5.25	8.29263206523851\\
};
\addplot [color=blue,mark size=0.5pt,only marks,mark=x,mark options={solid,draw=black},forget plot]
  table[row sep=crcr]{%
1	0\\
};
\addplot [color=blue,mark size=0.5pt,only marks,mark=x,mark options={solid,draw=black},forget plot]
  table[row sep=crcr]{%
1.25	0\\
};
\addplot [color=blue,mark size=0.5pt,only marks,mark=x,mark options={solid,draw=black},forget plot]
  table[row sep=crcr]{%
2	1.98521820039279\\
};
\addplot [color=blue,mark size=0.5pt,only marks,mark=x,mark options={solid,draw=black},forget plot]
  table[row sep=crcr]{%
2.25	1.8557762561928\\
};
\addplot [color=blue,mark size=0.5pt,only marks,mark=x,mark options={solid,draw=black},forget plot]
  table[row sep=crcr]{%
3	4.09583548071032\\
};
\addplot [color=blue,mark size=0.5pt,only marks,mark=x,mark options={solid,draw=black},forget plot]
  table[row sep=crcr]{%
3.25	3.8909599104109\\
};
\addplot [color=blue,mark size=0.5pt,only marks,mark=x,mark options={solid,draw=black},forget plot]
  table[row sep=crcr]{%
4	6.30279414388284\\
};
\addplot [color=blue,mark size=0.5pt,only marks,mark=x,mark options={solid,draw=black},forget plot]
  table[row sep=crcr]{%
4.25	6.15815458220558\\
};
\addplot [color=blue,mark size=0.5pt,only marks,mark=x,mark options={solid,draw=black},forget plot]
  table[row sep=crcr]{%
5	8.48053899114163\\
};
\addplot [color=blue,mark size=0.5pt,only marks,mark=x,mark options={solid,draw=black},forget plot]
  table[row sep=crcr]{%
5.25	8.29263206523851\\
};
\addplot [color=blue,mark size=2.0pt,only marks,mark=+,mark options={solid,draw=red},forget plot]
  table[row sep=crcr]{%
nan	nan\\
};
\addplot [color=blue,mark size=2.0pt,only marks,mark=+,mark options={solid,draw=green},forget plot]
  table[row sep=crcr]{%
nan	nan\\
};
\addplot [color=blue,mark size=2.0pt,only marks,mark=+,mark options={solid,draw=red},forget plot]
  table[row sep=crcr]{%
nan	nan\\
};
\addplot [color=blue,mark size=2.0pt,only marks,mark=+,mark options={solid,draw=green},forget plot]
  table[row sep=crcr]{%
nan	nan\\
};
\addplot [color=blue,mark size=2.0pt,only marks,mark=+,mark options={solid,draw=red},forget plot]
  table[row sep=crcr]{%
nan	nan\\
};
\addplot [color=blue,mark size=2.0pt,only marks,mark=+,mark options={solid,draw=green},forget plot]
  table[row sep=crcr]{%
nan	nan\\
};
\addplot [color=blue,mark size=2.0pt,only marks,mark=+,mark options={solid,draw=red},forget plot]
  table[row sep=crcr]{%
nan	nan\\
};
\addplot [color=blue,mark size=2.0pt,only marks,mark=+,mark options={solid,draw=green},forget plot]
  table[row sep=crcr]{%
nan	nan\\
};
\addplot [color=blue,mark size=2.0pt,only marks,mark=+,mark options={solid,draw=red},forget plot]
  table[row sep=crcr]{%
nan	nan\\
};
\addplot [color=blue,mark size=2.0pt,only marks,mark=+,mark options={solid,draw=green},forget plot]
  table[row sep=crcr]{%
nan	nan\\
};
\end{axis}
\end{tikzpicture}%
}%
~\\[-2ex]\centering{\scriptsize(a)}
\end{minipage}
\begin{minipage}{.01\columnwidth}
\hfill
\end{minipage}
\begin{minipage}{0.48\columnwidth}{
%
%
\definecolor{mycolor1}{rgb}{1.00000,0.00000,0.00000}%
\definecolor{mycolor2}{rgb}{0.00000,1.00000,0.00000}%
\begin{tikzpicture}

\begin{axis}[%
width=\fwidth,
height=0.996313\fheight,
at={(0\fwidth,0\fheight)},
scale only axis,
unbounded coords=jump,
separate axis lines,
every outer x axis line/.append style={black},
every x tick label/.append style={font=\color{black}},
xmode=log,
xmin=0.875,
xmax=5.375,
xtick={1.125,2.125,3.125,4.125,5.125},
xticklabels={{$\infty$},{21.6},{15.6},{12.1},{9.6}},
xlabel={Noise (dB)},
xminorticks=true,
every outer y axis line/.append style={black},
every y tick label/.append style={font=\color{black}},
ymin=0,
ymax=12,
ylabel={RMSD (mm)},
legend columns=-1,
legend to name=simdataconvex2_legend,
label style={font=\scriptsize}, legend style={font=\footnotesize}, every x tick label/.append style={font=\tiny}, every y tick label/.append style={font=\tiny}
]
\addplot [color=mycolor1,solid,forget plot]
  table[row sep=crcr]{%
1	0\\
1	0\\
};
\addplot [color=mycolor2,solid,forget plot]
  table[row sep=crcr]{%
1.25	0\\
1.25	0\\
};
\addplot [color=mycolor1,solid,forget plot]
  table[row sep=crcr]{%
2	1.7884533909458\\
2	3.63251056945719\\
};
\addplot [color=mycolor2,solid,forget plot]
  table[row sep=crcr]{%
2.25	1.67952408320059\\
2.25	3.15467141971697\\
};
\addplot [color=mycolor1,solid,forget plot]
  table[row sep=crcr]{%
3	3.56515759004045\\
3	7.82500764360873\\
};
\addplot [color=mycolor2,solid,forget plot]
  table[row sep=crcr]{%
3.25	3.63418406321611\\
3.25	6.81623840176844\\
};
\addplot [color=mycolor1,solid,forget plot]
  table[row sep=crcr]{%
4	5.35602376538474\\
4	11.3567172190925\\
};
\addplot [color=mycolor2,solid,forget plot]
  table[row sep=crcr]{%
4.25	5.58279383674379\\
4.25	11.1871711267721\\
};
\addplot [color=mycolor1,solid,forget plot]
  table[row sep=crcr]{%
5	7.49105445433071\\
5	13.8192286023526\\
};
\addplot [color=mycolor2,solid,forget plot]
  table[row sep=crcr]{%
5.25	7.31910806726331\\
5.25	13.8192286023526\\
};
\addplot [color=mycolor1,solid,line width=3.0pt,forget plot]
  table[row sep=crcr]{%
1	0\\
1	0\\
};
\addplot [color=mycolor2,solid,line width=3.0pt,forget plot]
  table[row sep=crcr]{%
1.25	0\\
1.25	0\\
};
\addplot [color=mycolor1,solid,line width=3.0pt,forget plot]
  table[row sep=crcr]{%
2	2.17622113789785\\
2	2.92418850698182\\
};
\addplot [color=mycolor2,solid,line width=3.0pt,forget plot]
  table[row sep=crcr]{%
2.25	2.03924925760305\\
2.25	2.52394969747254\\
};
\addplot [color=mycolor1,solid,line width=3.0pt,forget plot]
  table[row sep=crcr]{%
3	4.5250714622109\\
3	5.94711881248926\\
};
\addplot [color=mycolor2,solid,line width=3.0pt,forget plot]
  table[row sep=crcr]{%
3.25	4.21346428906241\\
3.25	5.3785286230186\\
};
\addplot [color=mycolor1,solid,line width=3.0pt,forget plot]
  table[row sep=crcr]{%
4	6.90883092036882\\
4	8.6955160269142\\
};
\addplot [color=mycolor2,solid,line width=3.0pt,forget plot]
  table[row sep=crcr]{%
4.25	6.49192114577134\\
4.25	8.39371724740698\\
};
\addplot [color=mycolor1,solid,line width=3.0pt]
  table[row sep=crcr]{%
5	9.00796121188473\\
5	11.141315369778\\
};

\addplot [color=mycolor2,solid,line width=3.0pt]
  table[row sep=crcr]{%
5.25	8.79275524183172\\
5.25	10.9531778002718\\
};

\addplot [color=blue,only marks,mark=*,mark options={solid,fill=blue,draw=mycolor1},forget plot]
  table[row sep=crcr]{%
1	0\\
};
\addplot [color=blue,only marks,mark=*,mark options={solid,fill=blue,draw=mycolor2},forget plot]
  table[row sep=crcr]{%
1.25	0\\
};
\addplot [color=blue,only marks,mark=*,mark options={solid,fill=blue,draw=mycolor1},forget plot]
  table[row sep=crcr]{%
2	2.42363485823428\\
};
\addplot [color=blue,only marks,mark=*,mark options={solid,fill=blue,draw=mycolor2},forget plot]
  table[row sep=crcr]{%
2.25	2.23225324770536\\
};
\addplot [color=blue,only marks,mark=*,mark options={solid,fill=blue,draw=mycolor1},forget plot]
  table[row sep=crcr]{%
3	4.88818183063221\\
};
\addplot [color=blue,only marks,mark=*,mark options={solid,fill=blue,draw=mycolor2},forget plot]
  table[row sep=crcr]{%
3.25	4.59441456783594\\
};
\addplot [color=blue,only marks,mark=*,mark options={solid,fill=blue,draw=mycolor1},forget plot]
  table[row sep=crcr]{%
4	7.42250400568515\\
};
\addplot [color=blue,only marks,mark=*,mark options={solid,fill=blue,draw=mycolor2},forget plot]
  table[row sep=crcr]{%
4.25	7.13288523835455\\
};
\addplot [color=blue,only marks,mark=*,mark options={solid,fill=blue,draw=mycolor1},forget plot]
  table[row sep=crcr]{%
5	9.84861757416476\\
};
\addplot [color=blue,only marks,mark=*,mark options={solid,fill=blue,draw=mycolor2},forget plot]
  table[row sep=crcr]{%
5.25	9.45036167126423\\
};
\addplot [color=blue,mark size=0.5pt,only marks,mark=x,mark options={solid,draw=black},forget plot]
  table[row sep=crcr]{%
1	0\\
};
\addplot [color=blue,mark size=0.5pt,only marks,mark=x,mark options={solid,draw=black},forget plot]
  table[row sep=crcr]{%
1.25	0\\
};
\addplot [color=blue,mark size=0.5pt,only marks,mark=x,mark options={solid,draw=black},forget plot]
  table[row sep=crcr]{%
2	2.42363485823428\\
};
\addplot [color=blue,mark size=0.5pt,only marks,mark=x,mark options={solid,draw=black},forget plot]
  table[row sep=crcr]{%
2.25	2.23225324770536\\
};
\addplot [color=blue,mark size=0.5pt,only marks,mark=x,mark options={solid,draw=black},forget plot]
  table[row sep=crcr]{%
3	4.88818183063221\\
};
\addplot [color=blue,mark size=0.5pt,only marks,mark=x,mark options={solid,draw=black},forget plot]
  table[row sep=crcr]{%
3.25	4.59441456783594\\
};
\addplot [color=blue,mark size=0.5pt,only marks,mark=x,mark options={solid,draw=black},forget plot]
  table[row sep=crcr]{%
4	7.42250400568515\\
};
\addplot [color=blue,mark size=0.5pt,only marks,mark=x,mark options={solid,draw=black},forget plot]
  table[row sep=crcr]{%
4.25	7.13288523835455\\
};
\addplot [color=blue,mark size=0.5pt,only marks,mark=x,mark options={solid,draw=black},forget plot]
  table[row sep=crcr]{%
5	9.84861757416476\\
};
\addplot [color=blue,mark size=0.5pt,only marks,mark=x,mark options={solid,draw=black},forget plot]
  table[row sep=crcr]{%
5.25	9.45036167126423\\
};
\addplot [color=blue,mark size=2.0pt,only marks,mark=+,mark options={solid,draw=mycolor1},forget plot]
  table[row sep=crcr]{%
nan	nan\\
};
\addplot [color=blue,mark size=2.0pt,only marks,mark=+,mark options={solid,draw=mycolor2},forget plot]
  table[row sep=crcr]{%
nan	nan\\
};
\addplot [color=blue,mark size=2.0pt,only marks,mark=+,mark options={solid,draw=mycolor1},forget plot]
  table[row sep=crcr]{%
2.00746697219861	4.38960330388368\\
};
\addplot [color=blue,mark size=2.0pt,only marks,mark=+,mark options={solid,draw=mycolor2},forget plot]
  table[row sep=crcr]{%
2.24128978568764	3.30997770522568\\
};
\addplot [color=blue,mark size=2.0pt,only marks,mark=+,mark options={solid,draw=mycolor1},forget plot]
  table[row sep=crcr]{%
3.01603184646887	8.26553847098129\\
};
\addplot [color=blue,mark size=2.0pt,only marks,mark=+,mark options={solid,draw=mycolor2},forget plot]
  table[row sep=crcr]{%
3.24461879679314	8.26553847098129\\
};
\addplot [color=blue,mark size=2.0pt,only marks,mark=+,mark options={solid,draw=mycolor1},forget plot]
  table[row sep=crcr]{%
nan	nan\\
};
\addplot [color=blue,mark size=2.0pt,only marks,mark=+,mark options={solid,draw=mycolor2},forget plot]
  table[row sep=crcr]{%
4.24952156902406	11.3567172190925\\
};
\addplot [color=blue,mark size=2.0pt,only marks,mark=+,mark options={solid,draw=mycolor1},forget plot]
  table[row sep=crcr]{%
nan	nan\\
};
\addplot [color=blue,mark size=2.0pt,only marks,mark=+,mark options={solid,draw=mycolor2},forget plot]
  table[row sep=crcr]{%
nan	nan\\
};
\end{axis}
\end{tikzpicture}%
}%
~\\[-2ex]\centering{\scriptsize(b)}
\end{minipage}
\\
\scriptsize{
\begin{tabular}{c c}
 \ref{pgfplots:simconreg1} \texttt{conreg1} & \ref{pgfplots:simconreg2} \texttt{conreg2} \\
\end{tabular}
}
    \caption{Simulation results of the two convex programs (a) when the images are the projection of the complete 3D model and (b) when the images contain the projections of the 3D model as a proper subset. For each noise level, we do 50 experiments with different poses of the 3D model, with rotation ranging from $0^\circ$ to $90^\circ$. The RMSD gradually increases as the noise increases from no noise ($\infty$ dB) to 2 mm (9.6 dB).  \label{fig:fullsim}}
\end{figure}
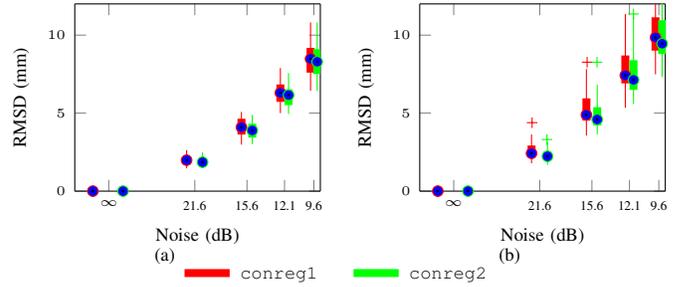

We observed exact recovery in both programs when the noise level is zero, which confirms our proof in Section \ref{section:Exact recovery}. Since our proposed methods are convex programs, the solution, in this case the ground truth rotation, does not depend on initialization. Fig.~\ref{fig:fullsimnoiseless} shows the results of both programs with no noise added to the points with respect to the initial pose of the 3D model. In our experiments, we set the tolerance level for the convex solver to $10^{-3}$. Fig.~\ref{fig:fullsim} shows the results for both programs with bounded noise added to each point in the image. In this case the RMSD increases proportionally as the noise $\varepsilon$ increases from 0 to 2 mm for both input conditions. 
We observe that the latter input condition of extra centerlines in the image may be a better approximation of a typical clinical scenario as often non vascular objects such as catheters are mis-segmented as vessels.

We compare our algorithm with the recent 2D/3D registration algorithm \texttt{OGMM} \cite{baka2014ogmm,dibildox2014ogmm}. In our experiments, we use our MATLAB implementation of \cite{baka2014ogmm} for two images with each individual gaussian distribution being isotropic. If we are to use \texttt{OGMM}, using the identity as an initialization, the recovery of the pose is not guaranteed unless we are in the $[0^\circ,10^\circ]$ range. In fact for most experiment we cannot find solution within 25 mm RMSD. We show the results from \texttt{OGMM} in Fig. \ref{fig:fullsimnoiseless}, where we plot the RMSD against different 3D poses. We note that there are many ways one can improve on our implementation of \texttt{OGMM} in terms of radius of convergence, such as annealing to zero the variance of Gaussian distributions starting from a large value as proposed in \cite{jian2011gmmreg}. Nevertheless, it still remains that registration methods based on Gaussian mixtures method \cite{jian2011gmmreg} can get stuck in a local optimum due to the nature of the cost.

\subsection{Real data}

\begin{figure}
\centering
  \setlength\fheight{2.5cm}
    \setlength\fwidth{0.18\textwidth}
\begin{minipage}{0.48\columnwidth}{\includegraphics{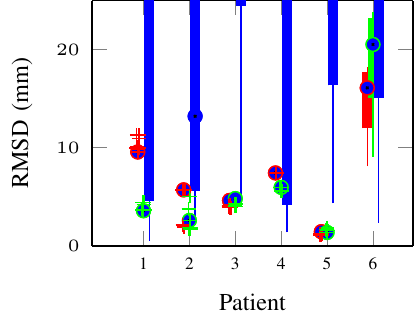}}%
~\\[-.9ex]\centering{\scriptsize(a)}
\end{minipage}
\begin{minipage}{.01\columnwidth}
\hfill
\end{minipage}
\begin{minipage}{0.48\columnwidth}{\includegraphics{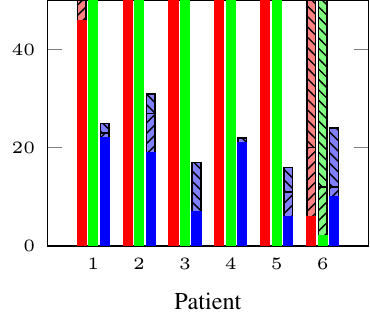}}
~\\[-.9ex]\centering{\scriptsize(b)}
\end{minipage}
\\
    \caption{(a) Results using the two convex programs \texttt{conreg1} (Red), \texttt{conreg2} (Green), and \texttt{OGMM} (Blue) for data of 6 patients. (b) Number of instances for which the RMSD error was less than 10 mm (solid color), 15 mm (north east lines) and 25 mm (north west lines), respectively. For each of the 6 patients, for each rotation range we generate 10 different instances resulting in a total of 50 instances per patient. }
    \label{fig:convex real}
\end{figure}

To illustrate the application of our algorithm on clinical data, we have used pre-processed images from Chronic Total Occlusion (CTO) cases treated using PCI. Both our methods \texttt{(REG1)} and \texttt{(REG2)} operate on point-sets and hence any images must be pre-processed to segment relevant structures such as vessels or organ surfaces. In case of CTO cases, the relevant structures to align are the centerlines of the coronaries as seen in pre-operative CTA images and intra-operative fluoroscopic images. To extract the centerlines in CTA images, we have used the segmentation algorithm presented in~\cite{zheng2013robust}. The resultant centerlines also provide topological information allowing us to compute geometric features such as tangents. These centerlines were re-sampled at uniform interval to result in the 3D point-set for registration methods. The fluoroscopic images were processed using 2D vessel segmentation algorithms~\cite{deschamps2001fast,sironi2014multiscale} to provide respective centerlines in the projection images. These were re-sampled to provide the 2D point-sets. The projection matrix was estimated from the C-arm parameters stored in the DICOM header of the image.

We test our algorithm on six sets of clinical data in which a medical expert has aligned the 3D model to 2D fluoroscopic images. We consider this the ground truth to compare against our results. To characterize the recovery of our algorithm, we perturb the aligned 3D model by an arbitrary rotation within the ranges $[0^\circ,10^\circ],[10^\circ, 20^\circ],[20^\circ, 40^\circ],[40^\circ, 60^\circ]$, and $[60^\circ, 90^\circ]$. For each rotation range, we generate 10 different instances. The results are detailed in figure~\ref{fig:convex real} and \ref{fig:convex real angles}. Again for comparison we use results from \texttt{OGMM}. We see that we are able to recover the pose to within 10 mm for most cases. Further, the recovery of pose in consistent irrespective of the initial pose. Whereas, in general we cannot obtain RMSD within 25 mm with \texttt{OGMM}. On average, there are 80 points in each model and image. The average running time for \texttt{conreg1} and \texttt{conreg2} are 0.4s and 15s, respectively on an Intel\textregistered\enspace CORE\texttrademark\enspace i7 2.2GHz running MATLAB\textregistered. \texttt{conreg2} is significantly slower than \texttt{conreg1} due to the additional positive semidefinite matrix that encodes cycle consistency. The effect of additional point consistency terms between the images in $\pb$ is more prominent in real patient data. We note that clinical data typically may have more number of ambiguous 3D model to projection point matches that cannot be resolved adequately using $\pa$ without enforcing explicit matching between images. In such cases, the point consistency terms in $\pb$ may aid in resolution of such ambiguities.
We present the images of the 3D model after the 2D/3D registration in Fig.~\ref{fig:realimage}.

\begin{figure}
 \centering
  \setlength\fheight{2.5cm}
    \setlength\fwidth{0.18\textwidth}
\begin{minipage}{0.48\columnwidth}{\includegraphics{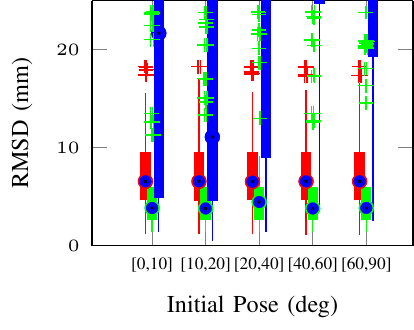}}%
~\\[-.9ex]\centering{\scriptsize(a)}
\end{minipage}
\begin{minipage}{.01\columnwidth}
\hfill
\end{minipage}
\begin{minipage}{0.48\columnwidth}{\includegraphics{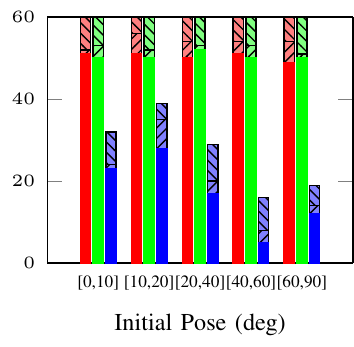}}
~\\[-.9ex]\centering{\scriptsize(b)}
\end{minipage}
\\
    \caption{(a) Results using the two convex programs \texttt{conreg1} (Red), \texttt{conreg2} (Green), and \texttt{OGMM} (Blue) for data of 6 patients with respect to angle of pose. (b) Number of instances for which the RMSD error was less than 10 mm (solid color), 15 mm (north east lines) and 25 mm (north west lines), respectively with respect to angle of pose. For each of the 6 patients, for each rotation range we generate 10 different instances resulting in a total of 50 instances per patient.}
    \label{fig:convex real angles}
\end{figure}

In Fig.~\ref{fig:realimage} we not only perform registration using \texttt{conreg1} and \texttt{conreg2} but also further refine their solutions using \texttt{gmmreg}. Typically for convex relaxation techniques such as the one adopted in this paper, it is not unusual to prove the convex relaxed solutions to the noisy data lie in the vicinity of the ground truth \cite{wang2013exact,candes2010matrix,chaudhury2013gret}. We, in fact, demonstrate via simulations with synthetic data that the solutions of algorithms \texttt{conreg1} and \texttt{congreg2} lie within certain radius (proportional to the noise magnitude) of the ground truth (Fig. \ref{fig:fullsim}). Thus, local optimizers such as \texttt{ICP} or GMM based methods after \texttt{conreg1} and \texttt{conreg2} could  be used to further refine the results. 
As local search based registration methods often fail to reach the global optimum due to non-convex nature of the cost, this combined approach ensures the ground truth is within the basin of convergence of any local methods. It is clear that this strategy brings improvement in row 1,3 and 4 in Fig.~\ref{fig:realimage} especially when registering using \texttt{conreg1}. While adding a local refinement step brings little changes visually in row 2 and 5 and even seems counterproductive for row 6 in Fig.~\ref{fig:realimage}, the RMSD when comparing with the ground truth given by the clinician is in fact lowered. For example, it is shown in Fig.~\ref{fig:convex real}a that \texttt{conreg1} and \texttt{conreg2} fail to achieve 15 mm RMSD for images of patient 6. However after a local refinement all test cases for 6 patients achieve 15 mm RMSD (Table.~\ref{table:summary}). The deterioration of registration quality for patient 6 in image plane A stems from the fact that the local refinement tries to fit the model to the vessels in both of the image planes. It may bring improvement to registration in one of the image plane while sacrificing the registration quality for the other image plane in order to lower the overall cost. Results combining algorithms \texttt{conreg1} and \texttt{congreg2} with \texttt{OGMM} and a variant of \texttt{LM-ICP} \cite{fitzgibbon2003robust} as local search methods are presented in Table~\ref{table:summary}. We include two different local methods to illustrate that the final results are not dependent on the local refinement method utilized. Though the desired accuracy of registration is highly dependent on application with some applications desiring sub-millimeter accuracy and others being satisfied with centimeters, in order to facilitate a comparison between different algorithms, we bin the RMSD error into ranges. As shown in Table~\ref{table:summary}, most cases fail to converge to within 25 mm when using the local methods alone. Our proposed method attain results well within the 10 mm for almost all cases, with sub-millimeter differences between the use of \texttt{OGMM} and \texttt{LM-ICP} as refinement methods. Though $\pb$ has some instances with larger RMSD error, overall its performance is better than $\pa$ in terms of median RMSD error. Both methods outperform using \texttt{OGMM} or the \texttt{LM-ICP} variant alone for robustness to initial starting point and median RMSD.

\newcommand\tstrut{\rule{0pt}{2.4ex}}
\newcommand\bstrut{\rule[-1.0ex]{0pt}{0pt}}
\begin{table}[h!]
\caption{The number of experiments with RMSD within a given range on 6 real data sets. Each patient data set was simulated with a total of 50 different poses, for 6 patients. The number in bracket is the median of the RMSD for solutions within respective thresholds. The range of RMSD is reported for attempts with RMSD greater than 25 mm.
\protect{\label{table:summary}}
}
{\small
\centering 
\begin{tabular}{c c c c} 
\toprule 
Methods & $<$10 mm & $<$ 15 mm & $>$25 mm\\
\midrule 
\footnotesize{\texttt{ICP}} & 19 (5.4) & 22 (5.71) & 269 [26.3, 237] \tstrut \bstrut \\
\footnotesize{\texttt{OGMM}} & 85 (4.57) & 101 (5.14) & 165 [25.2, 738] \tstrut \bstrut \\\midrule
\footnotesize{\texttt{conreg1}} & 252 (5.71) & 270 (5.71) &  \tstrut \bstrut \\
\footnotesize{\texttt{conreg1}+\texttt{OGMM}} & 300 (3.80) & 300 (3.80) &  \tstrut \bstrut \\
\footnotesize{\texttt{conreg1}+\texttt{ICP}}& 291 (4.40) & 299 (4.73)&  \tstrut \bstrut \\\midrule
\footnotesize{\texttt{conreg2}} & 262 (3.58) & 262 (3.58) &  \tstrut \bstrut \\
\footnotesize{\texttt{conreg2}+\texttt{OGMM}} & 291 (2.31) &  292 (2.31) & \tstrut \bstrut \\
\footnotesize{\texttt{conreg2}+\texttt{ICP}}& 293 (2.56) & 300 (2.68)&  \tstrut \bstrut \\
\bottomrule 
\end{tabular}
}
\end{table}

\begin{figure*}
\centering
\includegraphics[width=0.95\textwidth]{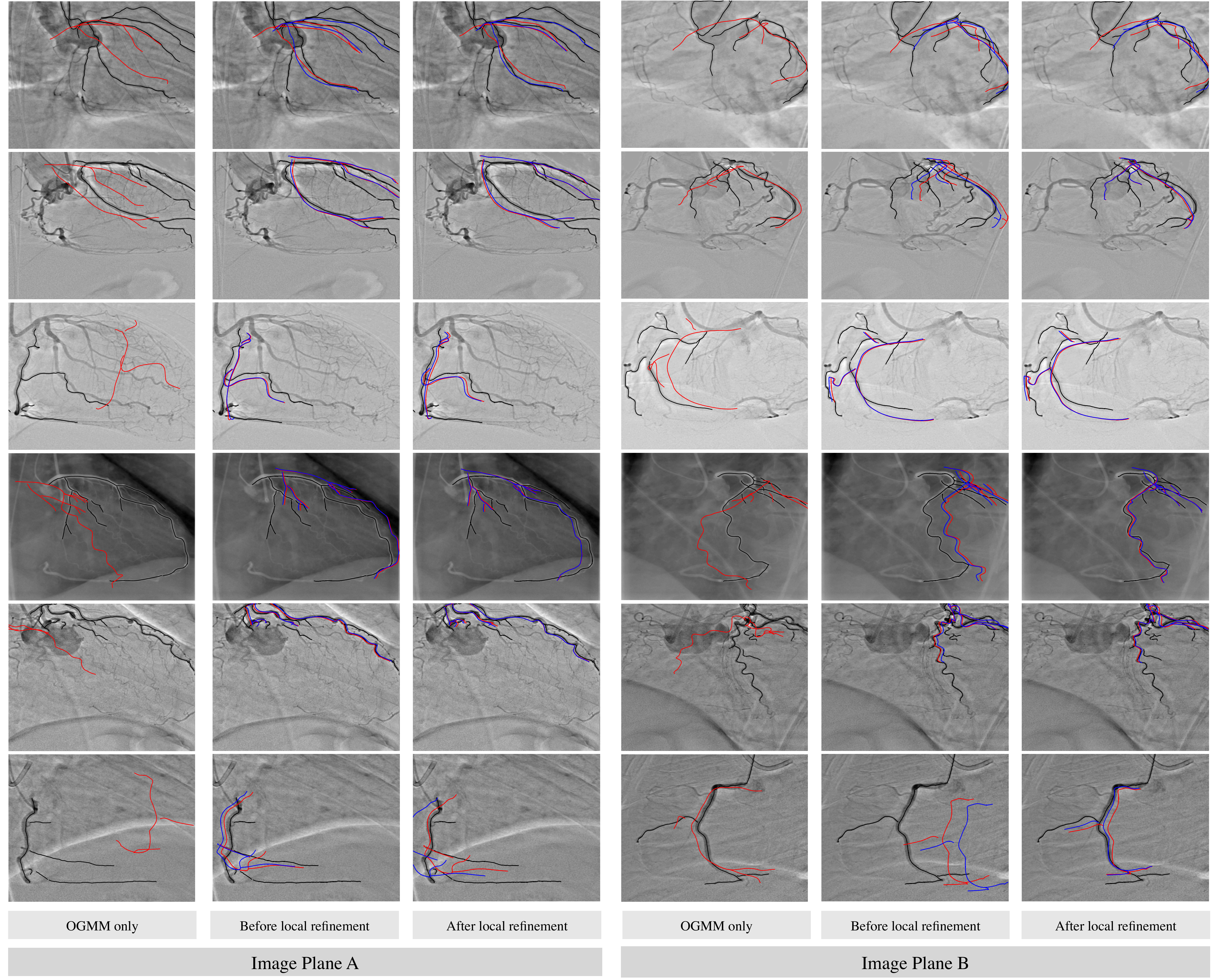}%
\caption{Each row is an illustrative sample result out of the 50 different initial poses used for registration for each single patient. The black lines are the segmented vessel centerline in the X-ray images. Column 1 and 4: Red lines denote the 3D model posed by \texttt{OGMM}. Column 2 and 5: Red lines denote the 3D model posed by $R^\star$ from $\pa$. Blue lines denote the 3D model posed by $R^\star$ from $\pb$. Column 3 and 6: Red lines denote the 3D model posed by $R^\star$ from $\pa$ + \texttt{OGMM}. Blue lines denote the 3D model posed by $R^\star$ from $\pb$ + \texttt{OGMM}. For brevity the initial pose is omitted as it moves the model out of the field of view. The results with \texttt{LM-ICP} as the local refinement method are omitted as they are visually not distinct from the results using \texttt{OGMM} method. As noted in Table~\ref{table:summary} the difference between the refinement from the two local methods is sub-millimeter. In this figure, the cases reported when using \texttt{OGMM} alone has error ranging from 12 mm to 40 mm. Except patient 6, the registration error of using the convex programs alone or with refinement are all below 10 mm. We refer the readers to Table \ref{table:summary} for a quantitative comparison of different methods.}
\label{fig:realimage}
\end{figure*}

\section{Discussion and Conclusion}
\label{section:Conclusion}
A new formulation for 2D/3D registration based on convex optimization programs has been proposed, and applied to the problem of registering a 3D centerline model of the coronary arteries with a pair of fluoroscopic images. The proposed optimization programs jointly optimize the correspondence between points and their projections, and the relative transformation. When the global optimum of the convex programs coincide with the solution to the 2D/3D registration problem, efficient search of the solution regardless of initialization can be done via standard off-the-shelf convex programming software.

In the first program presented, we find the correspondence and transformation simultaneously between two degenerate point clouds obtained by projection of a 3D model and the model itself. In the second program presented, we solve a variant of the first program where a-priori information on the estimated correspondence between degenerate point clouds is exploited. Such estimate of correspondence could come from feature matching or known epipolar constraints. We show that under certain conditions a convexly relaxed versions of these programs converge to recover the solution of the original MINLP point-set registration program. This is in contrast to \texttt{ICP}-type algorithms, where the correspondences and transformation are optimized alternatively. A natural extension of both programs is to consider additional descriptors and features attached to a particular point. We presented a framework to incorporate either transformation invariant features or features that transform according to rotation such as gradients. We illustrate this formulation with the use of tangent of the vessel derived from a neighborhood patch.

The proof of exact recovery for the noiseless case is presented and the theoretical  exact recovery results corroborated using synthetic data. We also evaluated the performance of the proposed formulations using noisy observations obtained by adding uniformly distributed bounded noise to the point sets. These exhibit convergence to within a radius of the exact solution that is proportional to the added noise level.
Both cases where the projected image fully captures the 3D model and cases when there are extra centerlines in the image were simulated. These mimic the clinical scenario where there may be surgical tools and devices such as catheters that have a similar appearance as that of a contrast filled vessel captured in the image. We find that under such conditions using local optimization programs such as \texttt{OGMM} alone, and using identity rotation as initialization, the recovery of the pose is not guaranteed unless we are in $0-10$ degree range from the optimal solution. In fact, for most experiments we cannot find a solution that converges to within 25 mm RMSD error. We note that there are ways one can improve on the implementations of local optimizers in terms of range of convergence by using methods such as annealing. Nevertheless, it still remains that registration methods based on local search may fail to reach the global optimum due to non-convex nature of the cost. We hope to bridge the gap in such cases, in addition to providing an explicit correspondence match.

Finally, experiments using multiple X-ray biplane angiography frames have also been presented. The validation was performed by perturbing the original pose by a random pose in a wide range of values. The pose and the correspondence can be recovered to within 10 mm RMSD error in most cases. Our algorithm can be used as a pre-processing step to provide a high quality starting point for a local registration algorithm such as \texttt{OGMM} or alone to provide recovery of transformation and correspondence.

There are several ways our proposed method can be improved. As shown in our proof, a condition we need for exact recovery is that the point sets are generic, implying the centroids of $X$ and $\I_i$'s are not zero. This limits our ability of including a translation by centering. A way to consider the translational degree of freedom is to include additional features. For example, we can use the tangent information in registration by introducing a cost in the form of $\|\Psi_i RT_X - T_{\I_i} P_i \|_{2,1}$, where $T_X$ and $T_{\I_i}$ in $\mathbb{R}^{3\times m}$ are the tangent vectors computed for points in the model and the image. In this case, with high probability the tangent vectors are not ``centered''  and we can still apply our proposed methods. The introduction of an additional cost term that encourages matching of the invariant features as in Section \ref{section:Additional features} can be helpful in similar way. We also note that during revision, a work in progress on solving the joint pose and correspondence problem using a tighter relaxation \cite{kezurer2015sdp} is brought to our attention. According to the communication with the authors there is no restriction in considering the translational degree of freedom in their relaxation. Another similar problem is the need of equal number of points in both the model and image in the proof. While it is possible for our methods to work when the number of points are different, the performance is less stable and it is important to identify extra features.

Since the focus of this paper is to introduce a novel formulation to the 2D/3D registration problem, we did not put an emphasis on designing a fast convex program solver. Currently we are using the \texttt{CVX} package, a library of interior point solvers for convex programming. Interior point method is inherently a second order method that requires high complexity computations involving the Hessian. It certainly does not scale as well as alternating minimization approaches such as \texttt{ICP} or first order methods such as \texttt{GMM}. Therefore it will be desirable to design fast first order optimization algorithms such as Augmented Lagrangian Method of Multiplier (ADMM) \cite{boyd2011admm} or conditional gradient descent \cite{frank1956fw} that have low complexity to solve \texttt{conreg1} and \texttt{conreg2}.

Currently we only have exact recovery guarantees when there is no noise in the image. As mentioned before, when data is plagued by noise the solution of \texttt{conreg1} and \texttt{conreg2} will no longer be the solution to (REG1) and (REG2). However, as observed in our numerical experiments, there is stability in the sense that $R^\star$ deviates gradually from the ground truth $\bar R$. A theoretical establishment of such stability will establish the convex programs as approximation algorithms for solving the registration problem.

Lastly, our proposal can only deal with rigid registration at this point. Since the consideration of non-rigid transformation will introduce further nonlinearity into the problem, a more elaborate convex relaxation has to be devised. Although we have not experimented with such scenarios, we recommend passing the solution from \texttt{conreg1} or \texttt{conreg2} to a non-rigid registration method based on local searches for refinement purpose.

\appendices
\section{}
\label{section:consistency}
\begin{proof}
The equations (\ref{I1 consistency}) and (\ref{I2 consistency}) simply follows from the fact that
\begin{eqnarray}
0 &\leq&  \| \Psi_1 \tilde R X' - \p_1 X' \tilde P_{1} \|_{2,1} + \| \Psi_2 \tilde R X'- \p_2 X' \tilde P_{2} \|_{2,1}\cr
  &\leq&  \| \Psi_1 I_3 X' - \p_1 X' I_m \|_{2,1} + \| \Psi_2 I_3 X'- \p_2 X' I_m \|_{2,1} \cr
  &=& 0.
\end{eqnarray}
The second inequality follows from the optimality of $\tilde P_{1}, \tilde P_{2}$ and $\tilde R$.

Now
\begin{eqnarray}
\tilde R X' &=& X' \tilde P_{1} + \eta_1\cr
\tilde R X' &=& X' \tilde P_{2} + \eta_2
\end{eqnarray}
where $\eta_1\in \mathbb{R}^{3\times m}$, $\eta_2 \in \mathbb{R}^{3\times m}$ satisfies $\p_1 \eta_1 = \p_2 \eta_2 = 0$. By induction, we have
\begin{eqnarray}
\tilde R^n X' &=& X' \tilde P^n_{1} + \eta_1(\tilde P_{1}^{n-1}+\ldots+ \tilde P_{1} + 1)\cr
\tilde R^n X' &=& X' \tilde P^n_{2} + \eta_2(\tilde P_{2}^{n-1}+\ldots+ \tilde P_{2} + 1).
\end{eqnarray}
Multiplying these two equations by $\p_1$ and $\p_2$ respectively we get (\ref{I1 power consistency}) and (\ref{I2 power consistency}).
\end{proof}

\section{}
\label{section:ergodicity}
\begin{proof}
Our goal is to show the solution $\tilde P_{1}, \tilde P_{2}, \tilde R$ being $I_m, I_m, I_3$ respectively. It suffices to prove uniqueness of $I_m,I_m, I_3$ as solution to $\pa$ in a local neighborhood, since for a convex program local uniqueness implies global uniqueness. Therefore we will assume $\|\tilde P_{1} - I_m \|_1< 1$ and $\|\tilde P_{2} - I_m \|_1< 1$. We now show $\tilde P_{1}, \tilde P_{2}$ are doubly stochastic matrices. The constraints $\mathbf{1}^T P_{1} = \mathbf{1}^T P_{2} = \mathbf{1}^T$ in $\pa$ indicates
\begin{equation}
\label{sum is m}
\sum_{i=1}^m (\sum_{j=1}^m \tilde P_{1}(i,j)) = \sum_{i=1}^m (\sum_{j=1}^m \tilde P_{2}(i,j)) = m.
\end{equation}
Since for every $i = 1,\ldots,m$, $\sum_{j=1}^m P_{1}(i,j) , \sum_{j=1}^m P_{2}(i,j)\leq 1$, it has to be that
\begin{equation}
\sum_{j=1}^m \tilde P_{1}(i,j) , \ \sum_{j=1}^m \tilde P_{2}(i,j) = 1 \ \ \ \mathrm{for\ every}\ i = 1,\ldots,m,
\end{equation}
or else it violates (\ref{sum is m}).

Based on these facts and regarding $\tilde P_1, \tilde P_2$ as transition matrices of a Markov chain, we will use the fundamental theorem of Markov chain to prove the lemma. Before stating the theorem, we need to give two standard definitions of stochastic processes.
\begin{definition}
A Markov chain with transition matrix $P\in \mathbb{R}^{m\times m}$ is aperiodic if the period of every state is 1, where the period of a state $i$, $i=1,\ldots,m$ is defined as
\begin{equation}
\mathrm{gcd}\{k: P^k(i,i)>0\}
\end{equation}
\end{definition}
\begin{definition}
A Markov chain with transition matrix $P\in \mathbb{R}^{m\times m}$ is irreducible if for all $i,j$, there exists some $k$ such that $P^k(i,j)>0$. Equivalently, if we regard $P$ as the adjacency matrix of a directed graph, it means the graph corresponding $P$ is strongly connected.
\end{definition}

We then state the fundamental theorem of Markov chain.
\begin{theorem}[Fundamental Theorem of Markov Chain \cite{dasgupta2011probability}]
If a Markov chain with transition matrix $P\in \mathbb{R}^{m\times m}$ is irreducible and aperiodic, then
\begin{equation}
\lim_{n\rightarrow \infty} P^n(i,j) = \pi_j\   \forall i = 1,\ldots,m,
\end{equation}
where the limiting distribution $\pi = [\pi_1,\ldots, \pi_m]$ is the unique solution to the equation
\begin{equation}
\pi = \pi P.
\end{equation}
\end{theorem}

We are now ready to prove the lemma. From the assumption $\|\tilde P_{1} - I_m \|_1< 1$ and $\|\tilde P_{2} - I_m \|_1< 1$, we have
\begin{equation}
\label{aperiodic}
\mathrm{diag}(\tilde P_{1})>0,\  \mathrm{diag}(\tilde P_{2})>0.
\end{equation}
This means the two doubly stochastic matrices $\tilde P_{1}$, $\tilde P_{2}$ are aperiodic. Furthermore we can decompose the matrices $\tilde P_{1}, \tilde P_{2}$ into irreducible components. Such decomposition essentially amounts to decomposing a graph into connected components, if we regard a doubly stochastic matrix as weighted adjacency matrix of a graph with $m$ nodes. Now denote $a_1, a_2,\ldots, a_{k}$ and  $b_1, b_2,\ldots, b_{l}$ as index sets of the irreducible components of matrix $\tilde P_{1}$ and $\tilde P_{2}$ respectively. For such decomposition, we know $\tilde P_{1}(a_i,a_j) = 0$ for $i\neq j$ (similarly for $\tilde P_2$). Therefore for each aperiodic and irreducible block $\tilde P_{1}(a_i,a_i)$ of $\tilde P_{1}$ and $\tilde P_{2}(b_i,b_i)$ of $\tilde P_{2}$, we have
\begin{eqnarray}
\label{limiting distribution}
(1/ \vert a_i \vert) \mathbf{1}^T \tilde P_{1}(a_i,a_i) &=& (1/ \vert a_i \vert) \mathbf{1}^T,\cr
(1/ \vert b_i \vert) \mathbf{1}^T \tilde P_{2}(b_i,b_i) &=& (1/ \vert b_i \vert) \mathbf{1}^T.
\end{eqnarray}
Applying fundamental theorem of Markov chain to each irreducible and aperiodic block along with equation (\ref{limiting distribution}), we have
\begin{eqnarray}
\label{ergodicblock}
\lim_{n\to\infty} \tilde P_{1}^n (a_i,a_i) &=& (1/ \vert a_i \vert) \mathbf{1} \mathbf{1} ^T,\cr
\lim_{n\to\infty} \tilde P_{2}^n (b_i,b_i) &=& (1/ \vert b_i \vert) \mathbf{1} \mathbf{1} ^T
\end{eqnarray}
and zero for all other indices. We note that the multiplication $X_{a_i}(1/ \vert a_i \vert) \mathbf{1} \mathbf{1}^T $ results $\vert a_i\vert$ copies of the centroid of $X_{a_i}$. Hence the limit of each irreducible component of $P_{1}, P_{2}$ is an averaging operator.

We now take limit of equations (\ref{I1 power consistency}) and (\ref{I2 power consistency}). Multiplying both (\ref{I1 power consistency}) and (\ref{I2 power consistency}) from the right with $[0\ 0\ 1]$, and using the fact that $[0\ 0\ 1] \p_1 = [0\ 0\ 1] \p_2 = [0\ 0\ 1]$  we get
\begin{equation}
\label{zaxisconsistency}
[0\ 0\ 1] X \lim_{n\to\infty} \tilde P_{1}^n = [0\ 0\ 1]  X \lim_{n\to\infty} \tilde P_{2}^n.
\end{equation}
Now assuming the point set $X$ contains generic coordinates, we know that no two set of points from $X$ have centers with same $z$ coordinates. Combining this fact with (\ref{ergodicblock}) and (\ref{zaxisconsistency}), we get $\{a_1, a_2,\ldots, a_{k}\} = \{b_1, b_2,\ldots, b_{l}\}$ and $\lim_{n\to\infty} P_{1}^n = \lim_{n\to\infty} P_{2}^n = A$.

\end{proof}

\section*{Acknowledgment}
The authors thank Ramon van Handel at Princeton University for the discussion of the proof of exact recovery. We also would like to acknowledge Professor Dr G.S.Werner, Klinikum Darmstadt, Darmstadt, DE for the image data and thank Andreas Meyer, Siemens AG for discussions related to clinical applications of our algorithms.

\ifCLASSOPTIONcaptionsoff
  \newpage
\fi

\bibliographystyle{IEEEtran}

\bibliography{./IEEEabrv,./bibref}

\begin{thebibliography}{10}
\providecommand{\url}[1]{#1}
\csname url@samestyle\endcsname
\providecommand{\newblock}{\relax}
\providecommand{\bibinfo}[2]{#2}
\providecommand{\BIBentrySTDinterwordspacing}{\spaceskip=0pt\relax}
\providecommand{\BIBentryALTinterwordstretchfactor}{4}
\providecommand{\BIBentryALTinterwordspacing}{\spaceskip=\fontdimen2\font plus
\BIBentryALTinterwordstretchfactor\fontdimen3\font minus
  \fontdimen4\font\relax}
\providecommand{\BIBforeignlanguage}[2]{{%
\expandafter\ifx\csname l@#1\endcsname\relax
\typeout{** WARNING: IEEEtran.bst: No hyphenation pattern has been}%
\typeout{** loaded for the language `#1'. Using the pattern for}%
\typeout{** the default language instead.}%
\else
\language=\csname l@#1\endcsname
\fi
#2}}
\providecommand{\BIBdecl}{\relax}
\BIBdecl

\bibitem{hipwell2003neuro}
J.~Hipwell \emph{et~al.}, ``Intensity-based {2D - 3D} registration of cerebral
  angiograms,'' \emph{Medical Imaging, IEEE Transactions on}, vol.~22, no.~11,
  pp. 1417--1426, Nov 2003.

\bibitem{benameur2003ortho}
S.~Benameur \emph{et~al.}, ``{3D/2D} registration and segmentation of scoliotic
  vertebrae using statistical models,'' \emph{Computerized Medical Imaging and
  Graphics}, vol.~27, no.~5, pp. 321--337, 2003.

\bibitem{ruijters2009coronary}
D.~Ruijters, B.~M. ter Haar~Romeny, and P.~Suetens, ``Vesselness-based {2D-3D}
  registration of the coronary arteries,'' \emph{International journal of
  computer assisted radiology and surgery}, vol.~4, no.~4, pp. 391--397, 2009.

\bibitem{markelj2012review}
P.~Markelj \emph{et~al.}, ``A review of {3D/2D} registration methods for
  image-guided interventions,'' \emph{Medical image analysis}, vol.~16, no.~3,
  pp. 642--661, 2012.

\bibitem{besl1992icp}
P.~J. Besl and N.~D. McKay, ``Method for registration of 3{D} shapes,'' in
  \emph{Robotics-DL tentative}.\hskip 1em plus 0.5em minus 0.4em\relax
  International Society for Optics and Photonics, 1992, pp. 586--606.

\bibitem{rusinkiewicz2001ICP}
S.~Rusinkiewicz and M.~Levoy, ``Efficient variants of the {ICP} algorithm,'' in
  \emph{3-D Digital Imaging and Modeling, 2001. Proceedings. Third
  International Conference on}.\hskip 1em plus 0.5em minus 0.4em\relax IEEE,
  2001, pp. 145--152.

\bibitem{sharp2002icp}
G.~C. Sharp, S.~W. Lee, and D.~K. Wehe, ``{ICP} registration using invariant
  features,'' \emph{Pattern Analysis and Machine Intelligence, IEEE
  Transactions on}, vol.~24, no.~1, pp. 90--102, 2002.

\bibitem{Benseghir2013}
\BIBentryALTinterwordspacing
T.~Benseghir, G.~Malandain, and R.~Vaillant,
  ``\BIBforeignlanguage{English}{Iterative closest curve: A framework for
  curvilinear structure registration application to {2D/3D} coronary arteries
  registration},'' in \emph{\BIBforeignlanguage{English}{Medical Image
  Computing and Computer-Assisted Intervention – MICCAI 2013}}, ser. Lecture
  Notes in Computer Science.\hskip 1em plus 0.5em minus 0.4em\relax Springer
  Berlin Heidelberg, 2013, vol. 8149, pp. 179--186. [Online]. Available:
  \url{http://dx.doi.org/10.1007/978-3-642-40811-3_23}
\BIBentrySTDinterwordspacing

\bibitem{rivest2012nonrigid}
D.~Rivest-Henault, H.~Sundar, and M.~Cheriet, ``Nonrigid {2D/3D} registration
  of coronary artery models with live fluoroscopy for guidance of cardiac
  interventions,'' \emph{Medical Imaging, IEEE Transactions on}, vol.~31,
  no.~8, pp. 1557--1572, 2012.

\bibitem{zheng2006point}
G.~Zheng \emph{et~al.}, ``Point similarity measures based on {MRF} modeling of
  difference images for spline-based {2D-3D} rigid registration of {X-ray}
  fluoroscopy to {CT} images,'' in \emph{Biomedical Image Registration}.\hskip
  1em plus 0.5em minus 0.4em\relax Springer, 2006, pp. 186--194.

\bibitem{kang2014robustness}
X.~Kang \emph{et~al.}, ``Robustness and accuracy of feature-based single image
  {2D-3D} registration without correspondences for image-guided intervention,''
  \emph{Biomedical Engineering, IEEE Transactions on}, vol.~61, no.~1, pp.
  149--161, 2014.

\bibitem{lau2006global}
K.~Lau and A.~C. Chung, ``A global optimization strategy for {3D}-{2D}
  registration of vascular images.'' in \emph{BMVC}, 2006, pp. 489--498.

\bibitem{chui2000nonrigid}
H.~Chui and A.~Rangarajan, ``A new algorithm for non-rigid point matching,'' in
  \emph{Computer Vision and Pattern Recognition, 2000. Proceedings. IEEE
  Conference on}, vol.~2.\hskip 1em plus 0.5em minus 0.4em\relax IEEE, 2000,
  pp. 44--51.

\bibitem{myronenko2010point}
A.~Myronenko and X.~Song, ``Point set registration: Coherent point drift,''
  \emph{Pattern Analysis and Machine Intelligence, IEEE Transactions on},
  vol.~32, no.~12, pp. 2262--2275, 2010.

\bibitem{jian2011gmmreg}
B.~Jian and B.~C. Vemuri, ``Robust point set registration using {Gaussian}
  mixture models,'' \emph{Pattern Analysis and Machine Intelligence, IEEE
  Transactions on}, vol.~33, no.~8, pp. 1633--1645, 2011.

\bibitem{baka2014ogmm}
N.~Baka \emph{et~al.}, ``Oriented {Gaussian} mixture models for nonrigid
  2{D}/3{D} coronary artery registration,'' \emph{Medical Imaging, IEEE
  Transactions on}, vol.~33, no.~5, pp. 1023--1034, 2014.

\bibitem{li2007iccv}
H.~Li and R.~Hartley, ``The {3D-3D} registration problem revisited,'' in
  \emph{Computer Vision, 2007. ICCV 2007. IEEE 11th International Conference
  on}, Oct 2007, pp. 1--8.

\bibitem{lee2011intraoperative}
J.~Lee \emph{et~al.}, ``Intraoperative {3D} reconstruction of prostate
  brachytherapy implants with automatic pose correction,'' \emph{Physics in
  medicine and biology}, vol.~56, no.~15, p. 5011, 2011.

\bibitem{chen2014matchlift}
Y.~Chen, L.~J. Guibas, and Q.-X. Huang, ``Near-optimal joint object matching
  via convex relaxation,'' \emph{arXiv preprint arXiv:1402.1473}, 2014.

\bibitem{huang2013cycle}
Q.-X. Huang and L.~Guibas, ``Consistent shape maps via semidefinite
  programming,'' in \emph{Computer Graphics Forum}, vol.~32, no.~5.\hskip 1em
  plus 0.5em minus 0.4em\relax Wiley Online Library, 2013, pp. 177--186.

\bibitem{ben2001lectures}
A.~Ben-Tal and A.~Nemirovski, \emph{Lectures on modern convex optimization:
  analysis, algorithms, and engineering applications}.\hskip 1em plus 0.5em
  minus 0.4em\relax Siam, 2001, vol.~2.

\bibitem{cvx}
{CVX Research, Inc.}, ``{CVX}: Matlab software for disciplined convex
  programming, version 2.0,'' \url{http://cvxr.com/cvx}, Aug. 2012.

\bibitem{gb08}
M.~Grant and S.~Boyd, ``Graph implementations for nonsmooth convex programs,''
  in \emph{Recent Advances in Learning and Control}, ser. Lecture Notes in
  Control and Information Sciences, V.~Blondel, S.~Boyd, and H.~Kimura,
  Eds.\hskip 1em plus 0.5em minus 0.4em\relax Springer-Verlag Limited, 2008,
  pp. 95--110, \url{http://stanford.edu/~boyd/graph_dcp.html}.

\bibitem{saunderson2014convhull}
J.~Saunderson, P.~A. Parrilo, and A.~S. Willsky, ``Semidefinite descriptions of
  the convex hull of rotation matrices,'' \emph{arXiv preprint
  arXiv:1403.4914}, 2014.

\bibitem{scheinerman2011fractional}
E.~R. Scheinerman and D.~H. Ullman, \emph{Fractional graph theory: a rational
  approach to the theory of graphs}.\hskip 1em plus 0.5em minus 0.4em\relax
  Courier Dover Publications, 2011.

\bibitem{candes2009lowrank}
E.~J. Cand{\`e}s and B.~Recht, ``Exact matrix completion via convex
  optimization,'' \emph{Foundations of Computational mathematics}, vol.~9,
  no.~6, pp. 717--772, 2009.

\bibitem{chaudhury2013gret}
K.~N. Chaudhury, Y.~Khoo, and A.~Singer, ``Global registration of multiple
  point clouds using semidefinite programming,'' \emph{arXiv preprint
  arXiv:1306.5226}, 2013.

\bibitem{biswas2006snl}
P.~Biswas \emph{et~al.}, ``Semidefinite programming approaches for sensor
  network localization with noisy distance measurements,'' \emph{Automation
  Science and Engineering, IEEE Transactions on}, vol.~3, no.~4, pp. 360--371,
  2006.

\bibitem{connelly2005generic}
R.~Connelly, ``Generic global rigidity,'' \emph{Discrete \& Computational
  Geometry}, vol.~33, no.~4, pp. 549--563, 2005.

\bibitem{sanyal2011orbitopes}
R.~Sanyal, F.~Sottile, and B.~Sturmfels, ``Orbitopes,'' \emph{Mathematika},
  vol.~57, no.~02, pp. 275--314, 2011.

\bibitem{brualdi1984doubly}
R.~A. Brualdi, ``The doubly stochastic matrices of a vector majorization,''
  \emph{Linear algebra and its Applications}, vol.~61, pp. 141--154, 1984.

\bibitem{hidehiko2013}
\BIBentryALTinterwordspacing
H.~Hara \emph{et~al.}, \emph{{CT} angiography}.\hskip 1em plus 0.5em minus
  0.4em\relax John Wiley {\&} Sons, 2013, pp. 43--50. [Online]. Available:
  \url{http://dx.doi.org/10.1002/9781118542446.ch6}
\BIBentrySTDinterwordspacing

\bibitem{rolf2013jcvi}
\BIBentryALTinterwordspacing
A.~Rolf \emph{et~al.}, ``\BIBforeignlanguage{English}{Preprocedural coronary
  {CT} angiography significantly improves success rates of {PCI} for chronic
  total occlusion},'' \emph{\BIBforeignlanguage{English}{The International
  Journal of Cardiovascular Imaging}}, vol.~29, no.~8, pp. 1819--1827, 2013.
  [Online]. Available: \url{http://dx.doi.org/10.1007/s10554-013-0258-y}
\BIBentrySTDinterwordspacing

\bibitem{cheung2010}
\BIBentryALTinterwordspacing
S.~C.~W. Cheung, M.~C.~L. Lim, and C.~W.~S. Chan, ``The role of coronary {CT}
  angiography in chronic total occlusion intervention,'' \emph{Heart Asia},
  vol.~2, no.~1, pp. 122--125, 2010. [Online]. Available:
  \url{http://heartasia.bmj.com/content/2/1/122.abstract}
\BIBentrySTDinterwordspacing

\bibitem{qu2014}
X.~Qu \emph{et~al.}, ``Clinical significance of a single multi-slice {CT}
  assessment in patients with coronary chronic total occlusion lesions prior to
  revascularization,'' \emph{PloS one}, vol.~9, no.~6, p. e98242, 2014.

\bibitem{zheng2013robust}
Y.~Zheng, H.~Tek, and G.~Funka-Lea, ``Robust and accurate coronary artery
  centerline extraction in {CTA} by combining model-driven and data-driven
  approaches,'' in \emph{Medical Image Computing and Computer-Assisted
  Intervention--MICCAI 2013}.\hskip 1em plus 0.5em minus 0.4em\relax Springer,
  2013, pp. 74--81.

\bibitem{deschamps2001fast}
T.~Deschamps and L.~D. Cohen, ``Fast extraction of minimal paths in {3D} images
  and applications to virtual endoscopy,'' \emph{Medical image analysis},
  vol.~5, no.~4, pp. 281--299, 2001.

\bibitem{sironi2014multiscale}
A.~Sironi, V.~Lepetit, and P.~Fua, ``Multiscale centerline detection by
  learning a scale-space distance transform,'' in \emph{Computer Vision and
  Pattern Recognition (CVPR), 2014 IEEE Conference on}.\hskip 1em plus 0.5em
  minus 0.4em\relax IEEE, 2014, pp. 2697--2704.

\bibitem{buades2005nlm}
A.~Buades, B.~Coll, and J.-M. Morel, ``A non-local algorithm for image
  denoising,'' in \emph{Computer Vision and Pattern Recognition, 2005. CVPR
  2005. IEEE Computer Society Conference on}, vol.~2.\hskip 1em plus 0.5em
  minus 0.4em\relax IEEE, 2005, pp. 60--65.

\bibitem{lowe1999sift}
D.~G. Lowe, ``Object recognition from local scale-invariant features,'' in
  \emph{Computer vision, 1999. The proceedings of the seventh IEEE
  international conference on}, vol.~2.\hskip 1em plus 0.5em minus 0.4em\relax
  Ieee, 1999, pp. 1150--1157.

\bibitem{mosekcomplexity}
E.~Andersen, ``Complexity of solving conic quadratic problems,''
  \url{http://erlingdandersen.blogspot.com/2013/11/complexity-of-solving-conic-quadratic.html},
  2013.

\bibitem{dibildox2014ogmm}
G.~Dibildox \emph{et~al.}, ``{3D/3D} registration of coronary {CTA} and biplane
  {XA} reconstructions for improved image guidance,'' \emph{Medical physics},
  vol.~41, no.~9, p. 091909, 2014.

\bibitem{wang2013exact}
L.~Wang and A.~Singer, ``Exact and stable recovery of rotations for robust
  synchronization,'' \emph{Information and Inference}, p. iat005, 2013.

\bibitem{candes2010matrix}
E.~J. Candes and Y.~Plan, ``Matrix completion with noise,'' \emph{Proceedings
  of the IEEE}, vol.~98, no.~6, pp. 925--936, 2010.

\bibitem{fitzgibbon2003robust}
A.~W. Fitzgibbon, ``Robust registration of 2{D} and 3{D} point sets,''
  \emph{Image and Vision Computing}, vol.~21, no.~13, pp. 1145--1153, 2003.

\bibitem{kezurer2015sdp}
I.~Kezurer, S.~Kovalsky, and Y.~Lipman, \emph{private communication}, August
  2015.

\bibitem{boyd2011admm}
S.~Boyd \emph{et~al.}, ``Distributed optimization and statistical learning via
  the alternating direction method of multipliers,'' \emph{Foundations and
  Trends{\textregistered} in Machine Learning}, vol.~3, no.~1, pp. 1--122,
  2011.

\bibitem{frank1956fw}
M.~Frank and P.~Wolfe, ``An algorithm for quadratic programming,'' \emph{Naval
  research logistics quarterly}, vol.~3, no. 1-2, pp. 95--110, 1956.

\bibitem{dasgupta2011probability}
A.~DasGupta, \emph{Probability for statistics and machine learning:
  fundamentals and advanced topics}.\hskip 1em plus 0.5em minus 0.4em\relax
  Springer, 2011.

\end{thebibliography}

\begin{IEEEbiography}[{\includegraphics[width=1in,height=1.25in,clip,keepaspectratio]{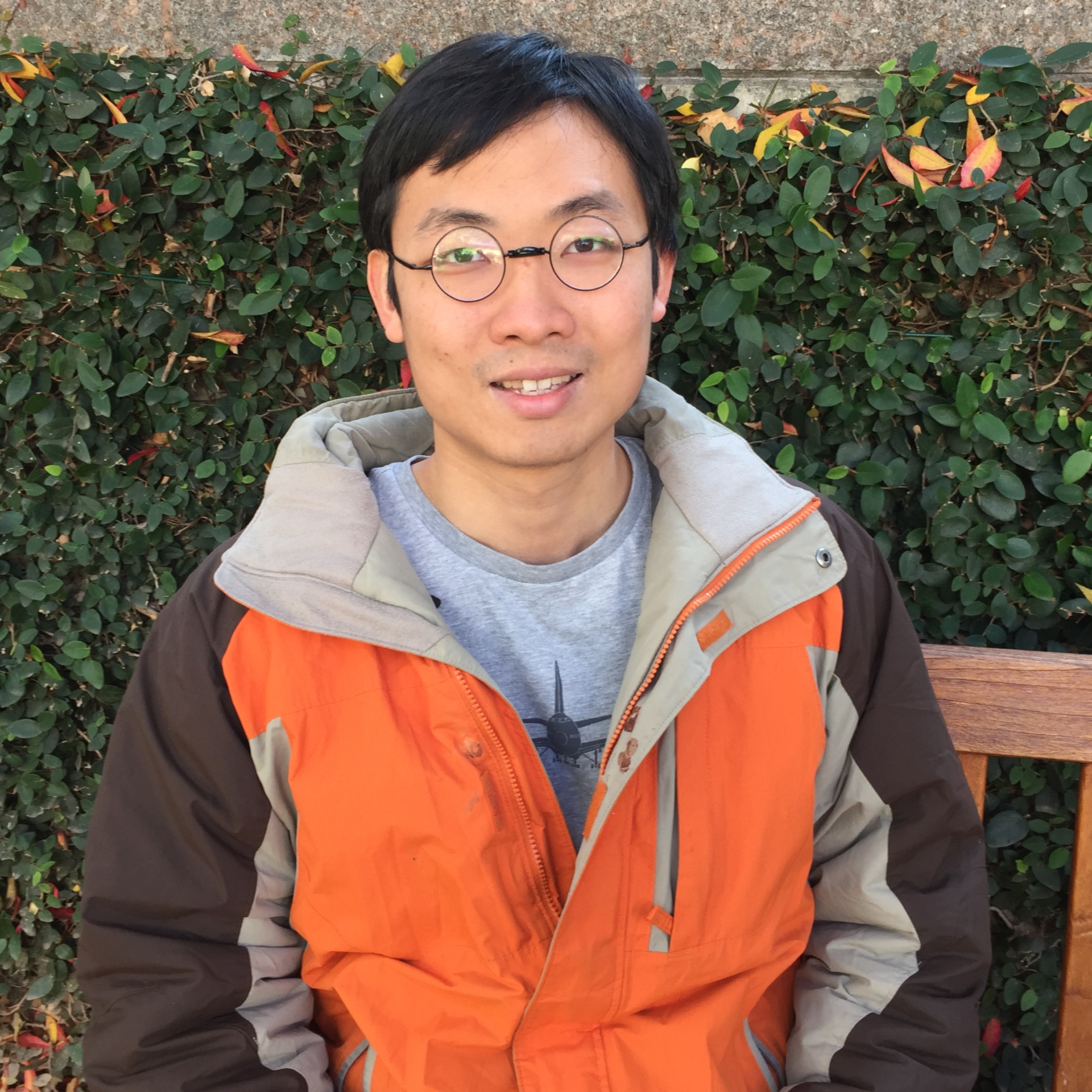}}]%
{Yuehaw Khoo} is borned in Malaysia. He is currently Ph.D. degree candidate in physics at Princeton University, advised by Professor Amit Singer. Previously he obtained B.Sc. degree (with highest distinction) in physics in 2009 at University of Virginia .
He worked at Siemens Research Corporation, Princeton, New Jersey, USA as a summer intern in year 2014. In 2009-2010 he worked as a research intern in the lab of Professor Gordon Cates at University of Virginia. He is interested in the application of convex programming relaxations to protein structural calculation and computer vision problems.
\end{IEEEbiography}

\begin{IEEEbiography}[{\includegraphics[width=1in,height=1.25in,clip,keepaspectratio]{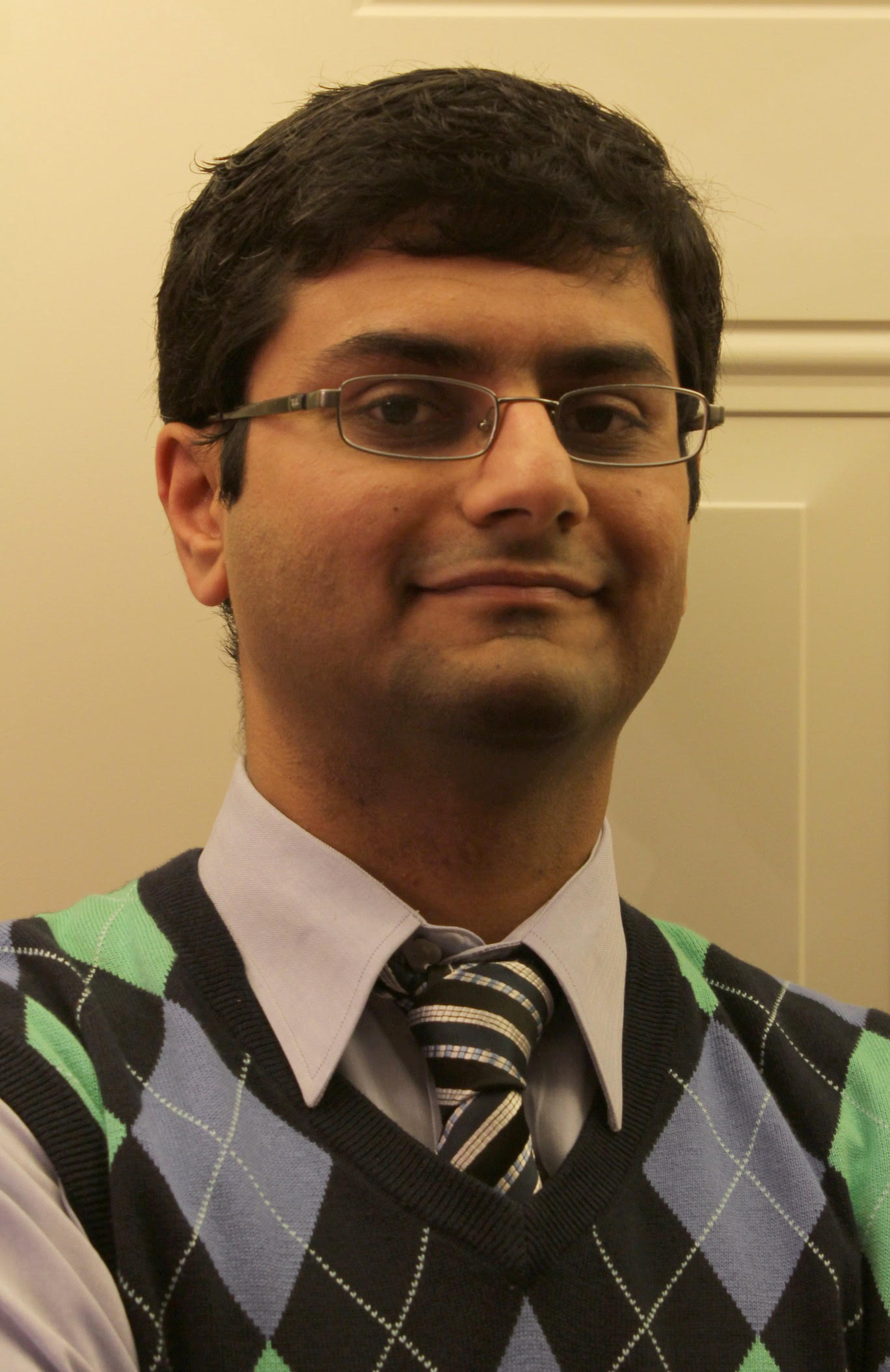}}]{Ankur Kapoor} (S’04–M’07) received the Bachelor of Engineering degree (with highest honors) in mechanical
and electrical and electronics from the Birla Institute of Technology and Science, Pilani, India, in 2000, and the M.S degree in computer science
and the Ph.D. degree in computer science from the Johns Hopkins University, Baltimore, MD, in 2006 and 2007, respectively.
He was a Research Fellow at Radiology and Imaging Sciences, Clinical Center, National Institutes of Health, Bethesda, MD till 2011. He currently a scientist at Medical Imaging Technologies, Siemens Healthcare. His research interests include robot assisted minimally invasive surgery, architecture for surgical systems, and image or visually guided interventions and surgery.
\end{IEEEbiography}

\end{document}